\documentclass{article}
\pdfoutput=1

\usepackage[utf8]{inputenc} 
\usepackage[T1]{fontenc}    
\PassOptionsToPackage{hyphens}{url} 
\usepackage[hyperfootnotes=false]{hyperref}
\usepackage{booktabs}       
\usepackage{amsfonts}       
\usepackage{nicefrac}       
\usepackage{microtype}      
\usepackage[table]{xcolor}  
\usepackage{breqn}
\usepackage[hang,flushmargin]{footmisc}
\usepackage{algorithm}
\usepackage{algpseudocode}
\usepackage{graphicx}
\usepackage{enumerate}
\usepackage{caption}
\usepackage{array}           
\usepackage{subcaption}
\usepackage{comment}
\usepackage[export]{adjustbox}
\usepackage{mathtools}
\usepackage{amsthm}
\usepackage{authblk}
\usepackage{fullpage}

\usepackage{enumitem}

\usepackage[round]{natbib}

\usepackage[HTML]{xcolor}
\usepackage[most]{tcolorbox}
\usepackage{Definitions}
\usepackage{dsfont}

\newcommand{\estb}{\bar{b}}
\newcommand{\vecb}{\vec{b}}

\newcommand{\estell}{\bar{\ell}}
\newcommand{\estmu}{\bar{\mu}}

\title{Learning to Decide with AI Assistance under Human-Alignment}

\author{Nina {Corvelo Benz}$^{\S}$}
\author{Eleni Straitouri$^{\dagger}$}
\author{Manuel~Gomez-Rodriguez$^{\dagger}$}
\affil{$^{\S}$Max Planck Institute for Biochemistry, Martinsried, Germany \\
corvelo@biochem.mpg.de}

\affil{$^{\dagger}$Max Planck Institute for Software Systems, Kaiserslautern, Germany \\ \{estraitouri, manuelgr\}@mpi-sws.org}

\date{}

\begin{document}

\maketitle

\begin{abstract}
%
It is widely agreed that when AI models assist decision-makers in high-stakes domains by predicting an outcome of interest, they should communicate the confidence of their predictions.
%
However, empirical evidence suggests that decision-makers often struggle to determine when to trust a prediction based solely on this communicated confidence.
%
In this context, recent theoretical and empirical work suggests a positive correlation between the utility of AI-assisted decision-making and the degree of alignment between the AI confidence and the decision-makers' confidence in their own predictions.
%
Crucially, these findings do not yet elucidate the extent to which this alignment influences the complexity of learning to make optimal decisions through repeated interactions.
In this paper, we address this question in the canonical case of binary predictions and binary decisions.
%
We first show that this problem is equivalent to a binary online contextual learning problem with full feedback, and establish a lower bound of $\Omega ( \sqrt{|\Hcal| \cdot |\Bcal| \cdot T} )$ 
on the expected regret any learner can attain, where $\Hcal$ and $\Bcal$ denote the sets of human and AI confidence values.
%
We then demonstrate that, under perfect alignment between AI and human confidence, a learner can attain an expected regret of $\mathcal{O}(\sqrt{|\Hcal| \cdot T\log T})$ and, when $\sqrt{|\Hcal|} = \mathcal{O}(\log T)$ and $\Bcal$ is countable, a non-trivial generalization of the Dvoretzky-Kiefer-Wolfowitz inequality improves the regret bound to $\mathcal{O}(\sqrt{T\log T})$.
Taken together, these results reveal that alignment can reduce the complexity of learning to make decisions with AI assistance.
%
Experiments on real data from two different human-subject studies where participants solve simple decision-making tasks assisted by AI models show that our theoretical results are robust to violations of perfect alignment.
\end{abstract}

\section{Introduction}
\label{sec:intro}
%
The promise of AI-based decision support systems is that human experts using these systems will make \emph{better} decisions than human experts not using them.
However, multiple lines of empirical evidence across various domains suggest that this promise has not yet been reliably realized~\citep{yin2019understanding,zhang2020effect,suresh2020misplaced,lai2023towards}.

%
In the canonical case of binary predictions and binary decisions,~\cite{corvelo2023human, corvelo2025human} have recently argued both theoretically and empirically that the way in which AI models quan\-ti\-fy and communi\-cate confidence in their predictions is one of the reasons AI-assisted decision making falls short.
First, they have shown that, if an AI model uses calibrated estimates of the probability that the prediction is correct as AI confidence values---as commonly done~\citep{Gneiting2007,zadrozny2001obtaining, guo2017calibration,gupta2021distribution,huang2020tutorial,wang2022improving}---a (rational) human expert who places more (less) trust on predictions with higher (lower) AI confidence may never make provably optimal decisions.
Second, they have provided empirical evidence that closer alignment between AI and human confidence correlates with higher utility in AI-assisted decision making.

%
Crucially, the aforementioned results do not yet elucidate the extent to which the alignment between AI and human confidence influences the complexity of learning to make (optimal) decisions through repeated interactions~\citep{noti2023learning, buccinca2024towards, noti2025ai, he2025human}. 
In this paper, we start closing  this gap.

%
\xhdr{Our contributions}
%
%
We first show that, in the canonical case of binary predictions and binary decisions, the problem of learning to make decisions with AI assistance through repeated interactions is equivalent to a binary online contextual learning problem with full feedback~\citep{slivkins2019introduction}.
%
%
Using this equivalence, we establish a lower bound of $\Omega ( \sqrt{|\Hcal| \cdot |\Bcal| \cdot T} )$ on the expected regret \emph{any} learner can attain in solving the problem, where $\Hcal$ and $\Bcal$ denote the sets of human and AI confidence values.
%
%
Further, we derive a simple online learning algorithm that, under perfect alignment between AI and human confidence, is guaranteed to attain an expected regret of $\mathcal{O}(\sqrt{|\Hcal| \cdot T\log T})$. 
Additionally, when $\sqrt{|\Hcal|} = \mathcal{O}(\log T)$ and $\Bcal$ is countable, we show that a non-trivial generalization of the Dvoretzky-Kiefer-Wolfowitz (DKW) inequality~\citep{dvoretzky1956asymptotic,massart1990tight} improves the regret bound to $\mathcal{O}(\sqrt{T\log T})$.
Taken together, our theoretical results reveal that alignment between AI and human confidence can reduce the complexity of learning to make decisions with AI assistance.
%
%
Finally, we complement our theoretical results with experiments on real data from two different human-subject studies where participants solve simple decision-making tasks assisted by AI models~\citep{vodrahalli2022humans,corvelo2025human}. 
The experiments demonstrate that our theoretical results are robust to violations of perfect alignment---a learner who assumes perfect alignment can achieve lower expected regret than one making no such assumption, provided that a weaker notion of alignment holds.
We have released all code and data used in our experiments at: \href{https://github.com/Human-Centric-Machine-Learning/learning-under-alignment}{https://github.com/Human-Centric-Machine-Learning/learning-under-alignment}.

%
\xhdr{Further related work} Our work builds upon further related work on AI-assisted decision making, multi-armed bandits, selective prediction and learning under algorithmic triage.
%

%
Our work is motivated by several empirical studies within the rapidly increasing literature on AI-assisted decision making (see~\cite{lai2023towards} for a recent review), which show that decision-makers have difficulties in modulating trust using confidence 
values~\citep{vodrahalli2022uncalibrated, yona2022useful, straitouri2023improving, straitouri2025narrowing, straitouri2024designing, de2024towards, gondocs2025uncovering, devic2025calibration}. 
In this context, we should also note that other studies have analyzed how additional context, such as model explanations and accuracy, can help modulate trust~\citep{papenmeier2019model, wang2021explanations, yin2019understanding, nourani2020role, zhang2020effect, guo2025value, hullman2025explanations,mei2025can}.

%
%
Within the vast literature of multi-armed bandits (see~\cite{slivkins2019introduction} for a recent review) the DKW inequality has been used to derive regret bounds in settings in which the objective requires precise estimation of the entire reward distribution~\citep{chen2016combinatorial, kearns2017meritocratic, keramati2020being, baudry2021optimal, howard2022sequential}, such as the conditional value at risk (CVaR).
However, to the best of our knowledge, the DKW inequality has not been used to derive improved average regret bounds in contextual settings in which the optimal decision policy satisfies a monotonicity property with respect to the context as in our work.

There is a rich line of work on selective prediction, also referred to as predicting with the reject option~\citep{cordella1995method,de2000reject,fumera2002support,el2010foundations, cortes2016boosting, geifman2017selective, geifman2019selectivenet, hendrickx2024machine}. In selective prediction, the AI model has the option to abstain from (reject) a decision task whenever the confidence of the AI prediction is low. 
Algorithmic triage extends this paradigm by assigning the rejected decision tasks to a human while learning under algorithmic triage focuses on routing decision tasks to either a human or an AI model~\citep{raghu2019algorithmic, mozannar2020consistent, de2020regression, de2021classification, okati2021differentiable}. 
Our work focuses on a fundamentally different problem setting, where a human makes every decision, using the AI model strictly as an assistive tool.

\vspace{-1em}
\section{Learning to Decide with AI Assistance}
\label{sec:model}
In this section, we first revisit the AI-assisted decision making task we focus in our work, which has also been studied elsewhere~\citep{corvelo2023human, corvelo2025human}. Then, in the context of this task, we introduce the problem of learning an optimal decision policy through repeated interactions, and characterize its complexity.
\subsection{AI-Assisted Decision Making with Binary Predictions and Binary Decisions} 
\label{sec:setting}
We consider an AI-assisted decision making task where, for each realization of the task, 
a decision maker first observes a set of features $(x, v) \in \Xcal \times \Vcal$, 
then takes a binary decision $a \in \{0, 1\}$ informed by a classifier'{}s prediction $\hat y = \argmax_{y} f_y(x)$, 
as well as confidence $f_{\hat y}(x) \in \Bcal \subseteq [0, 1]$, of a binary label of interest $y \in \{0, 1\}$,
and finally receives a utility $u(a, y) \in \RR$.
Such an AI-assisted decision making process fits a variety of real-world applications. 
For example, in medical treatment, the features $(x, v)$ may comprise multiple sources of information regarding a patient'{}s health\footnote{Our formulation allows for a subset of the features $v$ to be available only to the decision maker but not to the classifier.}, 
the label $y$ may indicate whether a patient would benefit from a specific treatment,
the decision $a$ may indicate whether the doctor applies the specific treatment to the patient,
and the utility $u(a, y)$ may quantify the trade-off between health benefit to the patient and economic cost to the decision maker.

In what follows, rather than working with both $\hat y$ and $f_{\hat y}(x)$, we will work with just $b = f_{1}(x)$, which we will refer to as classifier'{}s confidence, without loss of generality\footnote{We can recover $\hat y$ and $f_{\hat y}(x)$ from $b$, \ie, if $b > 0.5$, we have that $\hat y = 1$ and $f_{\hat y}(x) = b$; if $b < 0.5$, $\hat y = 0$ and 
$f_{\hat y}(x) = 1-b$.}.
Moreover, we will assume that the gain from deciding $a=1$ when $y=1$ and the gain from deciding $a=0$ when $y=0$ is positive, \ie, 
\begin{equation} \label{eq:u_prop}
u(1,1)-u(0,1)>0 \quad \text{and} \quad u(0,0)- u(1,0)>0,
\end{equation}
%
%
a condition that it is natural under an appropriate choice of label and decision values.
For example, in medical diagnosis, if $a = 1$ means the patient is tested early for 
a disease and $y = 1$ means the patient suffers the disease, 
the above condition implies that, for a patient who suffers the disease, the utility of testing them is greater than the utility of not testing them and, for a patient who does not suffer the disease, the utility of not testing them is greater than the utility of testing them.

Further, for each task instance, we formally characterize the decision maker's decision $a$ through a decision policy $\pi(h, b) \in \{0, 1\}$, where $h \in \Hcal$ and $b \in \Bcal \subseteq [0,1]$ denote the decision maker'{}s confidence and the classifier'{}s confidence that the value of the label of interest is $y = 1$, respectively.\footnote{Similar to~\cite{corvelo2023human}, the decision maker's confidence $h$ refers to the confidence the decision maker has that the label $Y = 1$ \emph{before} observing the classifier's confidence $b$.}
%
%
Here, following previous behavioral studies showing that human's confidence $h$ is discretized in a few distinct levels~\citep{lisi2021discrete,zhang2015human}, we assume $\Hcal$ is a totally ordered finite discrete set.
Then, under this characterization, the problem of learning to make optimal decisions reduces to finding the optimal decision policy $\pi^*$ that maximizes the expected utility, \ie,
\begin{equation}~\label{eq:best-response-policy}
    \pi^*(h,b) = \arg\max_a \mu(a \given h, b) \,\, \text{where} \,\, \mu(a \given h, b) = \mathrm{E}_Y[u(a,Y) \given  H=h, B=b],
\end{equation}
where the expectation is over the randomness of the label of interest $Y$.\footnote{We use upper-case for random variables and lower case for their realizations.}

Next, we examine the extent to which a decision maker can efficiently learn the optimal policy $\pi^{*}$ through repeated interactions---employing a sequence of policies $\pi_t$ over subsequent task realizations.

\subsection{Complexity of Learning to Decide with AI Assistance}
\label{sec:lower-bound}

Our starting point is the realization that, for binary predictions and binary decisions, the problem of learning to decide with AI assistance through repeated interactions is equivalent to a binary online contextual learning problem with full feedback~\citep{slivkins2019introduction}. 
In particular, at each time step $t$, the confidence values $(h_t, b_t)$ correspond to the context, the decisions $a$ corresponds to the arms, the utility $u(a_t, y_t)$ corresponds to the reward, and we have full feedback because, given the label of interest $y_t$, the decision maker can compute both $u(a_{t}, y_{t})$ and $u(1-a_{t}, y_{t})$.\footnote{In our setting, the decision maker is the learner, hence, there is no elicitation of confidence.}
Such full feedback setting fits many high-stakes decision-making tasks well. For example, doctors order cultures for patients with infections to test for antimicrobial resistance, but they frequently initiate antibiotic treatment immediately before the results are finalized. Further, traders may decide to trade a few specific stocks, but they can observe the entire market’s evolution over time. 
That said, there also exist decision-making tasks that do not fit the full-feedback setting. Refer to Section~\ref{sec:discussion} for a more in-depth discussion.

Consequently, similarly as elsewhere in the online learning literature~\citep{slivkins2019introduction}, we can measure to which extent a decision maker succeeds at efficiently finding the optimal decision policy by analyzing the expected cumulative regret. 
More formally, the expected cumulative regret $\EE[R(T)]$
is defined as
\begin{equation}\label{eq:regret_contextual}
    \EE[R(T)]=  \EE_{h_t,b_t \sim P(H,B)} \left[ \sum_{t=1}^T \left( \mu(\pi^*(h_t,b_t)\given h_t,b_t) - \EE_{a_t\sim P(A_t \given h_t, b_t)} [ \mu(a_t\given h_t,b_t) ] \right) \right],
\end{equation}
where $P(H, B)$ denotes the distribution of the decision maker's and the classifier's confidence, which implicitly depends on the feature distribution $P(X, V)$, and $P(A_t \given h_t, b_t)$ denotes the distribution of decisions induced by the decision policy $\pi_t$ used at time step $t$. 

In what follows, we establish a lower bound on the expected regret using a non-trivial extension of a classical result for contextual bandits with partial feedback from~\citet{agarwal2012contextual} to our full feedback setting.\footnote{All proofs are provided in Appendix~\ref{app:proofs}.}
\begin{theorem}~\label{th:vanilla-lower-bound}
    For any choice of decision policies $\pi_t$ based on historical observations, there exists a problem instance such that $\EE[R(T)] = \Omega ( \sqrt{|\Hcal| \cdot |\Bcal| \cdot T})$.
\end{theorem} 
The above result reveals that, in general, the complexity of learning the optimal decision policy through repeated interactions depends on the size of the sets of possible human and AI confidence values, $\Hcal$ and $\Bcal$.\footnote{Our lower bound is only meaningful whenever both $\Hcal$ and $\Bcal$ are finite discrete sets.
Note that, without further assumptions on the decision policy, learning is impossible under infinite context~\citep{shalev2014understanding}.}
Importantly, in the next section, we will show that, under perfect alignment between AI and human confidence, this dependence vanishes, reducing the overall complexity.

\section{Learning to Decide with AI Assistance under Perfect Alignment}
\label{sec:learning}
In this section, we first show that, under perfect alignment between AI and human confidence, the optimal decision policy $\pi^{*}$ is a threshold function defined over the AI confidence, with a threshold that depends on the human confidence.
Next, we present a simple online algorithm that, by utilizing this structure, is guaranteed to attain an expected regret of $\mathcal{O}(\sqrt{|\Hcal| \cdot T\log T})$ in learning the optimal decision policy $\pi^{*}$.
Finally, when $\sqrt{|\Hcal|} = \mathcal{O}(\log T)$ and $\Bcal$ is countable, we improve the regret bound to $\mathcal{O}(\sqrt{T\log T})$.
In Appendix~\ref{app:approximate-alignment}, we further show that, whenever perfect alignment does not hold, the degree of alignment bounds the gap between the expected utility achieved by the decision policy found by our algorithm and the optimal policy.\footnote{The algorithms presented are not meant to replace the human decision maker, but to help us understand the complexity of learning to make decisions under AI-assistance.}

\subsection{Optimal Decision Policy Under Perfect Alignment}
There is perfect alignment between AI and human confidence if, for any $h, h' \in \Hcal$ and $b, b' \in \Bcal$ such that $h\leq h'$ and  $b\leq b'$, we have that~\citep{corvelo2023human}
\begin{equation}~\label{eq:alignment}
P(Y=1 \given H=h, B=b) - P(Y=1 \given H=h', B=b') \leq 0.
\end{equation}
Intuitively, perfect alignment between the classifier’s and human confidence means the odds of a positive outcome (Y=1) increase (monotonically) with both the AI and human confidence. Hence, under perfect alignment, if both the AI and human are more confident for instances in a group A than in another group B, it can never happen that the odds of a positive outcome in group B are higher than group A. 

For utility functions $u$ satisfying Eq.~\ref{eq:u_prop}, the following theorem shows that, for each $h \in \Hcal$, there is a threshold value $b^{*}(h)$ such that $\pi^{*} (h, b) = 1$ if $b > b^{*}(h)$ and $\pi^{*} (h, b) = 0$ otherwise.
\begin{theorem}~\label{th:optimal-threshold}
Under perfect alignment, for any $h \in \Hcal$, the optimal decision policy $\pi^*(h,b)$ is given by a threshold function $\pi^*(h,b)=\II[b >  b^*(h)]$ where 
     \begin{equation*}
     b^*(h) = \sup \left\{ b\in \Bcal \ \middle| \ P(Y=0 \given H=h, B=b) \geq \frac{u(1,1) - u(0,1)}{u(1,1) - u(1,0) + u(0,0) - u(0,1)} \right\}.
    \end{equation*}
\end{theorem}

Importantly, the above result reveals that learning the optimal policy essentially reduces to learning $b^{*}(h)$, $\forall h \in \Hcal$.
In what follows, we present a simple online algorithm to learn $b^{*}(h)$, $\forall h \in \Hcal$, and show that it attains an expected regret that no longer depends on $\Bcal$ and, when $\Bcal$ is countable and $\sqrt{|\Hcal|} = \Ocal(\log T)$, no longer depends on $\Hcal$.

\begin{algorithm}[t]
\caption{Online Learning under Perfect Alignment}\label{alg:bandit-alg}
\begin{algorithmic}
    \Require  \# interactions $T$, set of human confidence values $\Hcal$, set of AI confidence values $\Bcal$
    \State{Initialize $\estmu_1( b \given h)$, $\forall (h, b) \in \Hcal \times \Bcal$} 
    \For{$t$ from $1$ to $T$} 
    \State{Observe confidence values $(h_t, b_t)$ }
    \State{Compute unbiased estimates $\bar{\mu}_{t}(b \given h_{t}), \, \forall b \in \Bcal$ using Eq.~\ref{eq:def_estmu}}
    \State $\estb_t(h_t) \leftarrow \arg \max_{b \in \Bcal} \bar{\mu}_{t}(b \given h_{t})$
    \If{$b_t > \estb_t(h_t)$}
    \State $a_t \leftarrow 1$
    \Else
    \State $a_t \leftarrow0$
    \EndIf
    \State Take decision $a_t$ and observe label value $y_t$
    \EndFor
\end{algorithmic}
\end{algorithm}

\subsection{An Efficient Algorithm for Learning to Decide Under Perfect Alignment}
\label{sec:independent-thresholds}

Our starting point is the realization that the problem of learning $b^{*}(h)$, $\forall h \in \Hcal$ through repeated interactions is equivalent to $|\Hcal|$ independent multiclass online learning problems with full feedback---one per value of $h$.
For each of these problems, at each time step $t$, the (candidate) thresholds $b \in \Bcal$ correspond to the arms, the utility $u(\II[b_t>b], y_t)$ corresponds to the reward, and we have full feedback because, given the label of interest $y_t$, the decision maker can compute the utility $u(\II[b_t>b], y_t)$ for all thresholds $b\in \Bcal$.

To solve the above multiclass online learning problems, our online algorithm leverages the following key theorem:
\begin{theorem}\label{th:policy-reward}
    Let $\mu(b\given h)$ be the conditional expected utility achieved by the decision policy $\pi$ defined by a threshold function with threshold $b$ given the human confidence $H=h$, \ie,
    \begin{equation}\label{eq:policy-reward}
     \mu(b\given h) = \EE \left [\mu(1\given H,B) \cdot \mathbb{I}[ B> b] \given H=h \right] +\EE \left [\mu(0\given H,B) \cdot \mathbb{I}[ B\leq b]\given H=h \right],
    \end{equation}
    where $\mu(a \given H,B)$ is the conditional expected utility achieved by decision $a$ given the human and AI confidence $H$ and $B$, defined in Eq.~\ref{eq:best-response-policy}.
    Then, it holds that $b^*(h)= \arg\max_{b\in \Bcal} \mu(b \given h)$. 
\end{theorem}
In a nutshell, our online algorithm proceeds as follows.
At each time step $t$, it first computes an unbiased estimate of the conditional expected utility $\mu(b\given h_t)$ for each candidate threshold $b \in \Bcal$, \ie,
\begin{align} \label{eq:def_estmu}
    \bar{\mu}_{t}( b \given h_{t}) = \frac{1}{n_t(h_t)}\sum_{t'<t: h_{t'}=h_t} 
    &\left(\II[y_{t'}=0, b_{t'} > b] \cdot(u(1,0)-u(1,1)) \right. \nonumber \\
    &+ \II[y_{t'}=0, b_{t'} \leq b] \cdot(u(0,0)-u(0,1)) \nonumber \\ 
    &+ \left. \II[b_{t'} \leq b] \cdot (u(0,1)-u(1,1)) + u(1,1) \right).
\end{align}
where $n_t(h_t)$ is the number of observations with human confidence equal to $h_t$ until time step $t$.
Then, it picks the threshold $\bar{b}_t(h_t)$ with the highest estimated conditional expected utility $\bar{\mu}_t(b \given h_t)$. 
Finally, it takes decision $a_t = 1$ if $b_t > \bar{b}_t(h_t)$ and $a_t = 0$ otherwise.
Algorithm~\ref{alg:bandit-alg} summarizes the overall algorithm, where note that the $|\Hcal|$ independent multiclass online learning problems are solved simultaneously. 

The following theorem bounds the expected regret attained by our algorithm:
\begin{theorem}\label{th:regret-bound}
    Algorithm~\ref{alg:bandit-alg} attains expected regret $\EE[R(T)]= \mathcal{O} \left (\sqrt{|\Hcal|\cdot T \log T} \right)$ under perfect alignment.
\end{theorem}
The proof of the above theorem relies on two key steps. 
First, we show that, under perfect alignment, our algorithm's expected regret on the $|\Hcal|$ independent instances of the multiclass online learning problem is equivalent to its expected regret on the binary online contextual learning problem with full feedback introduced in Section~\ref{sec:lower-bound}.\footnote{Note that our algorithm can be used to solve the binary online contextual learning problem with full feedback introduced in Section~\ref{sec:lower-bound} because thresholds $b \in \Bcal$ determine decisions $a \in \{0,1\}$.}
Second, we use the DKW inequality~\citep{dvoretzky1956asymptotic,massart1990tight} to obtain a uniform bound on the estimation error of $\bar{\mu}(b \given h_t)$ across all $b$ values, which is the key to achieving a regret bound independent of the AI confidence $b$.

Next, when $\sqrt{|\Hcal|} = \mathcal{O}(\log T)$ and $\Bcal$ is countable, we derive a tighter regret bound, which is independent not only on the AI confidence $b$ but also on the human confidence $h$. 
The key idea is to reformulate our algorithm in terms of a function $\ell : \Hcal \rightarrow \Bcal$ rather than in terms of $|\Hcal|$ independent thresholds.
This reformulation leverages the following key theorem, which is the counterpart of Theorem~\ref{th:policy-reward}:
\begin{theorem}\label{th:policy-reward-2d}
    Let $\Lcal = \{ \ell(\cdot) \given \ell : \Hcal \to \Bcal\}$ and, for each $\ell \in \Lcal$, let $\mu(\ell)$ be the expected utility achieved by the decision policy $\pi(h, b) = \II[b > \ell(h)]$, \ie,
    \begin{equation}\label{eq:policy-reward-2d}
    \mu(\ell) := \EE \left [\mu(1\given H,B) \cdot \mathbb{I}[ B> \ell(H)] \right] +\EE \left [\mu(0\given H,B) \cdot \mathbb{I}[ B\leq \ell(H)]\right],
\end{equation}
    where $\mu(a \given H,B)$ is the conditional expected utility achieved by decision $a$ given the human and AI confidence $H$ and $B$, and let $\ell^{*} = \arg\max_{\ell \in \Lcal} \mu(\ell)$.
    Then, for all $h \in \Hcal$, it holds that $b^*(h) = \ell^{*}(h)$.
\end{theorem}
More specifically, under this reformulation, the algorithm proceeds as follows. At each time step $t$, it first computes an unbiased estimate of the expected utility $\mu(\ell)$, \ie,
\begin{multline}\label{eq:estimator-2d}
    \estmu_{t}(\ell) = P_{t}(Y=0, B > \ell(H)) \cdot(u(1,0)-u(1,1)) 
    + P_{t}(Y=0, B \leq \ell(H)) \cdot(u(0,0)-u(0,1)) \\ 
    + P_{t}(B \leq \ell(H)) \cdot (u(0,1)-u(1,1)) + u(1,1),
\end{multline}
where $P_t(E) := \sum_{t'< t} \II[E_{t'}]$.
Then, it picks the function $\estell_t \in \Lcal$ with the highest estimated expected utility $\estmu_{t}(\ell)$. 
Finally, it takes decision  $a_t = 1$ if $b_t > \estell_t(h_t)$ and $a_t = 0$ otherwise.
 
To see why the above reformulation of our algorithm is faithful to the original, note that
\begin{align}\label{eq:argmax_Lt}
    \max_{\ell \in \Lcal} \estmu_{t}(\ell) = \frac{1}{n} \sum_{h\in \Hcal} \quad \max_{\ell(h) \in \Bcal} n_t(h) \cdot \bar{\mu}_t( \ell(h) \given h ),
\end{align}
and, as an immediate consequence, it readily follows that $\bar{\ell}_t(h_t) = \estb_t(h_t)$ $\forall t$.

The following theorem bounds the expected regret attained by the above reformulation of our algorithm, which improves the bound in Theorem~\ref{th:regret-bound} by a factor of $\log T$ when $\sqrt{|\Hcal|} =\Ocal(\log T)$ and $\Bcal$ is countable:
\begin{theorem}\label{th:regret-bound-2d}
    Under perfect alignment, there exists constant $c>0$ such that, if $\sqrt{|\Hcal|} < c\cdot \log(T)$ and $\Bcal$ is countable, Algorithm~\ref{alg:bandit-alg} attains expected regret $\EE[R(T)]= \mathcal{O} \left (\sqrt{T\log T} \right)$.
\end{theorem}

The proof of the above theorem is analogous to the proof of Theorem~\ref{th:regret-bound} but relies on a non-trivial extension of the DKW inequality to obtain a uniform bound on the estimation error of $\bar{\mu}(\ell)$ across all $\ell \in \Lcal$, which is the key to achieving a regret bound that is not only independent of the AI confidence $b$ but also independent of the human confidence $h$ and may be of independent interest (see Appendix~\ref{app:dkw-like-inequality} and Theorem~\ref{thm:dkwi-thresholds} for more details). 

\xhdr{Remark}
A naive implementation of Algorithm~\ref{alg:bandit-alg} may require computing $|\Bcal|$ values of the estimated expected utility, one for each $b \in \Bcal$. 
However, if $\Bcal$ is large, such a naive implementation may be very inefficient and, if $\Bcal$ is a continuous set, it may be infeasible. 
To overcome this issue, we can observe that the estimated expected utility can only take at most $T+1$ different values at $T$ time steps. 
This is because, for any observed AI confidence values $b_t$ and $b_{t'}$ such that $b_t < b_{t'}$ and no other observed AI confidence value lies in between, all decision policies $\pi(b, h) = \one[b > \bar{b}_t(h)]$ using any threshold $\bar{b}_t(h)$ such that $b_{t'} \leq \bar{b}_t(h)  <b_{t}$ will take the same decision for any observed $b \leq b_{t'}$ or $b_{t } \leq b$, and thus have  
the same estimated expected utility value.

\section{Experiments}
\label{sec:experiments}
\begin{algorithm}[t]
\caption{Vanilla Contextual Online Learning}\label{alg:contextual-alg} 
\begin{algorithmic}
    \Require  \# interactions $T$, set of human confidence values $\Hcal$, set of AI confidence values $\Bcal$
    \State{Initialize $\estmu_1( a \given h, b)$, $\forall a \in \{0,1\}, \forall (h,b) \in \Hcal \times \Bcal$} 
    \For{each time step $t$ from $1$ to $T$} 
    \State{Observe confidence values $(h_t, b_t)$ }
    \For{$a \in \{0,1\}$}
    \State $\estmu_t( a \given h_t, b_t) \leftarrow  \frac{1}{|\{t'<t\,:\,(h_{t'},b_{t'})=(h_t,b_t)\}|} \sum_{t'<t\,:\,(h_{t'}, b_{t'})=(h_t,b_t)} \left( \II[y_{t'}=1] \cdot u(a,1) + \II[y_{t'}=0] \cdot u(a,0)\right)$
    \EndFor
    \If{$\estmu_t( 0 \given h_t, b_t) \geq \estmu_t( 1 \given h_t, b_t)$}
    \State $a_t \leftarrow 0$
    \Else
    \State $a_t \leftarrow 1$
    \EndIf
    \State Take action $a_t$ and observe label value $y_t$
    \EndFor
\end{algorithmic}
\end{algorithm}
\begin{figure}[ht!]
    \centering
    \includegraphics[width=.6\linewidth]{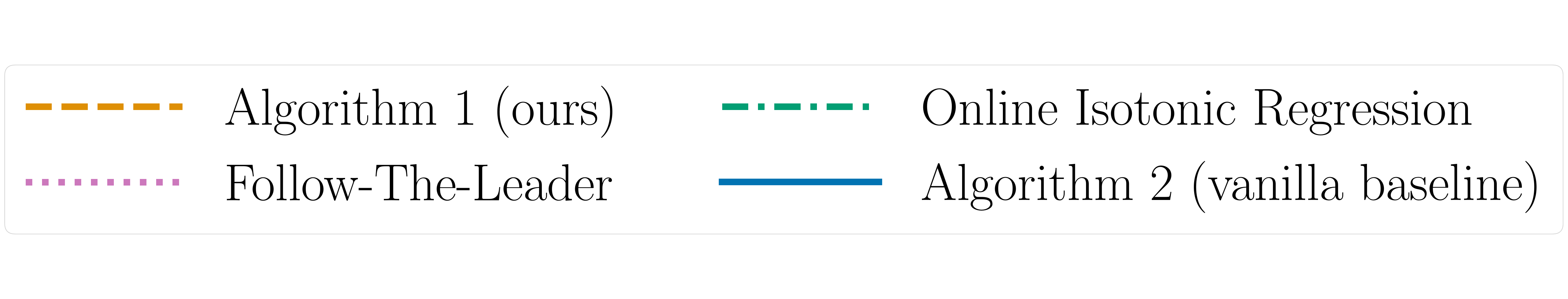}\\
    \subfloat[Group A]{
    \includegraphics[width=0.38\linewidth]{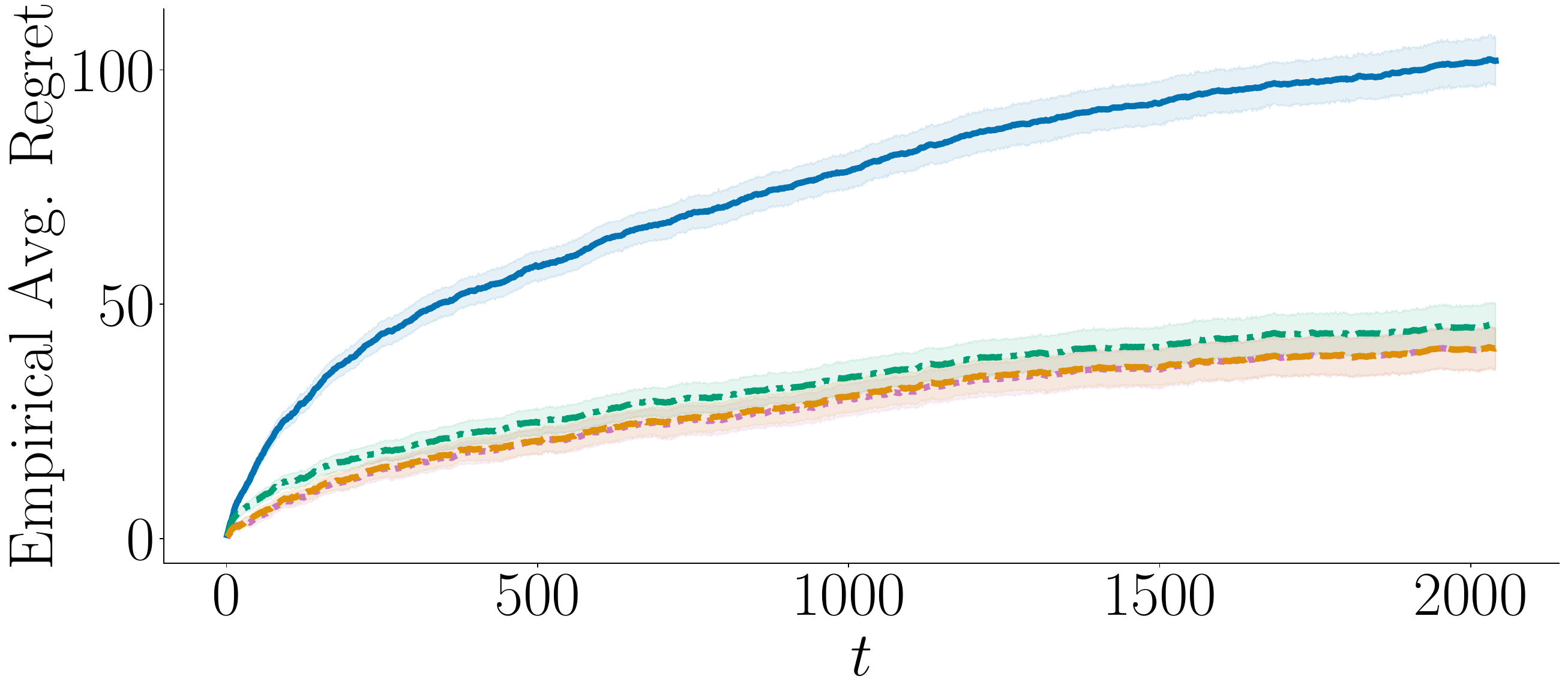}
    } \hspace{5mm}
    \subfloat[Art]{
    \includegraphics[width=0.38\linewidth]{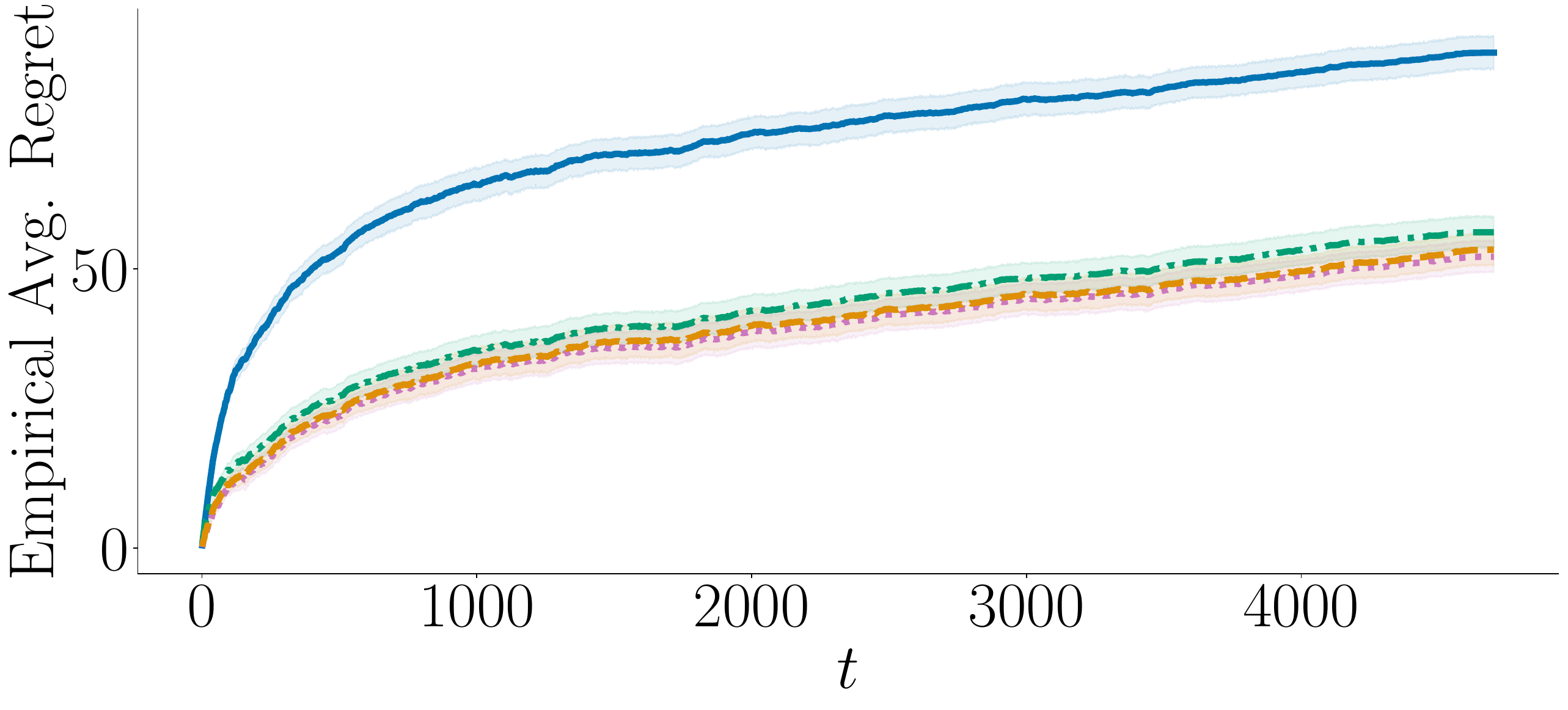}
    }\\
    
    \subfloat[Group B]{
    \includegraphics[width=0.38\linewidth]{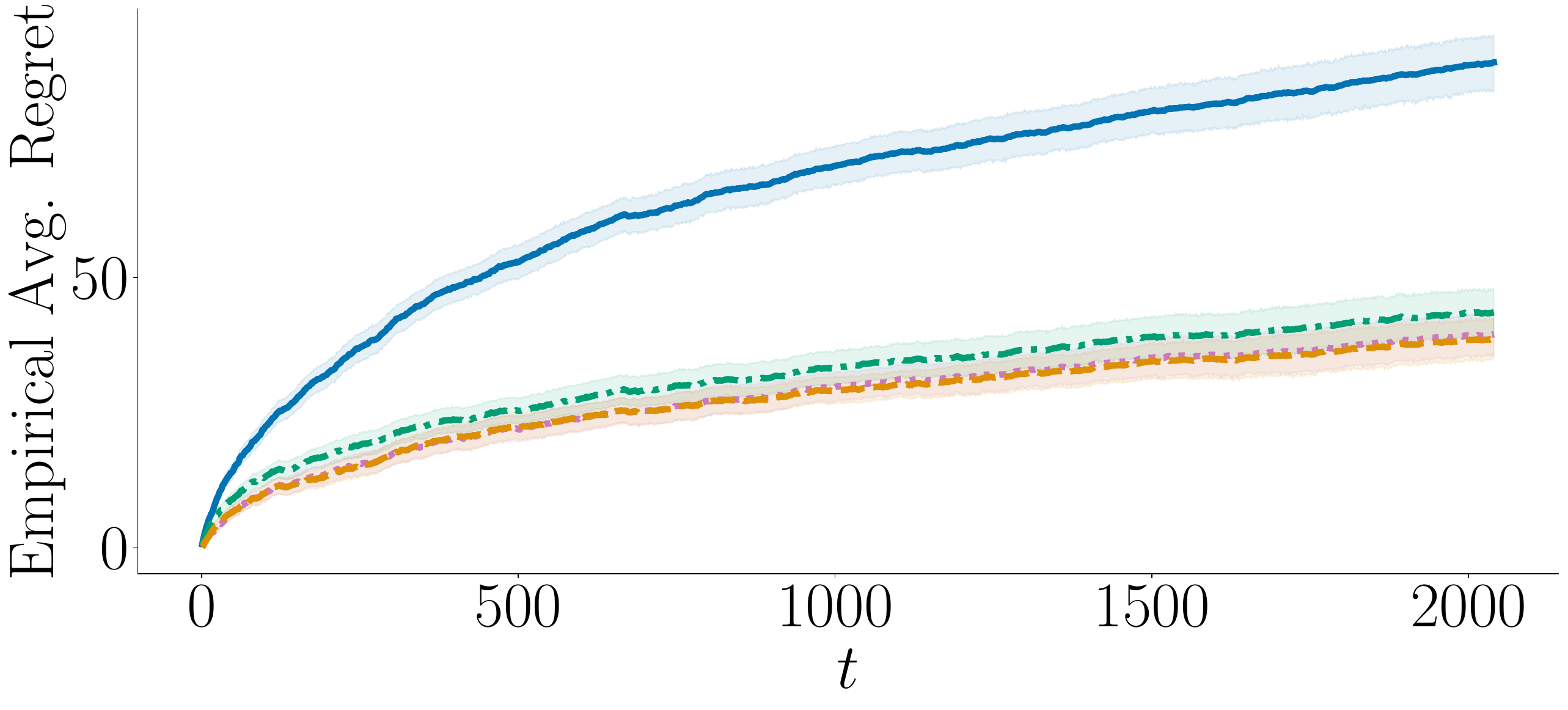}
    } \hspace{5mm}
    \subfloat[Census]{
    \includegraphics[width=0.38\linewidth]{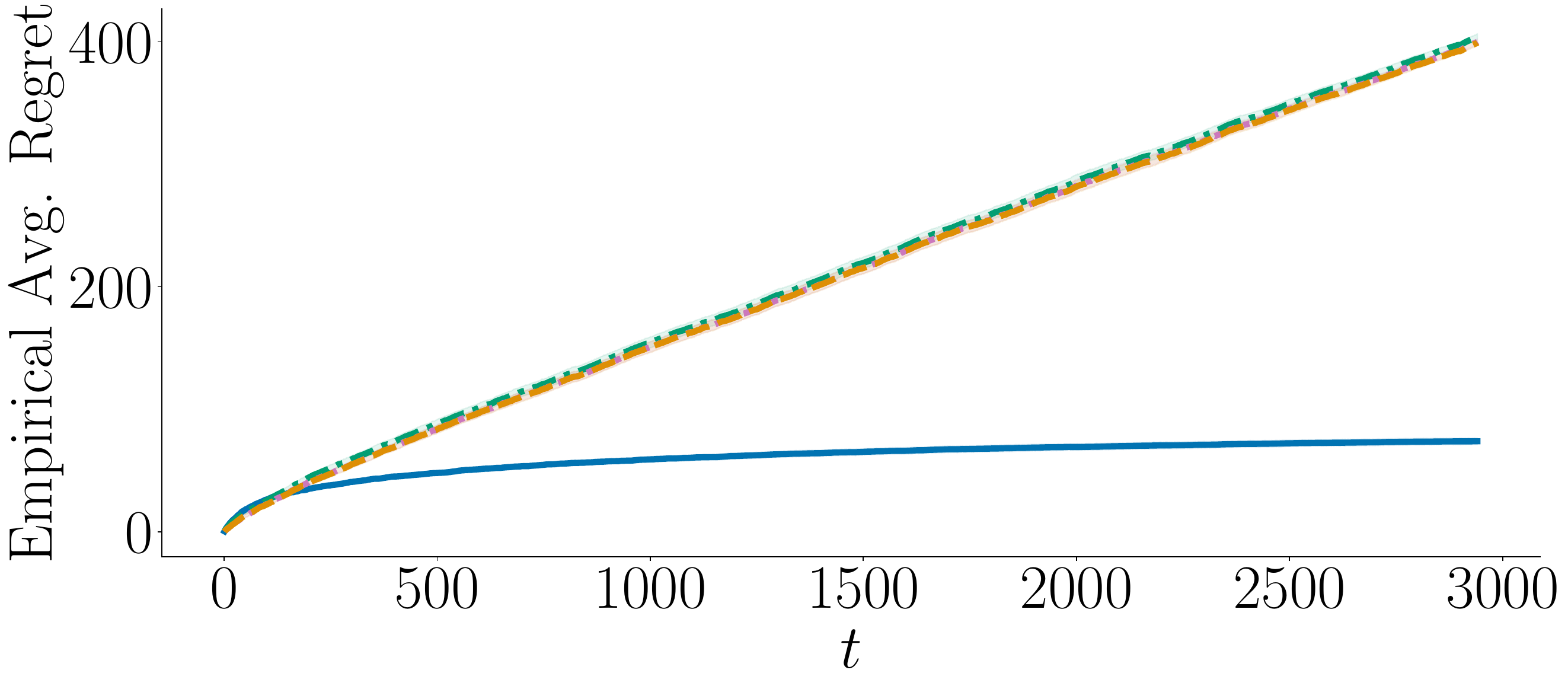}
    }\\
    
    \subfloat[Group C]{
    \includegraphics[width=0.38\linewidth]{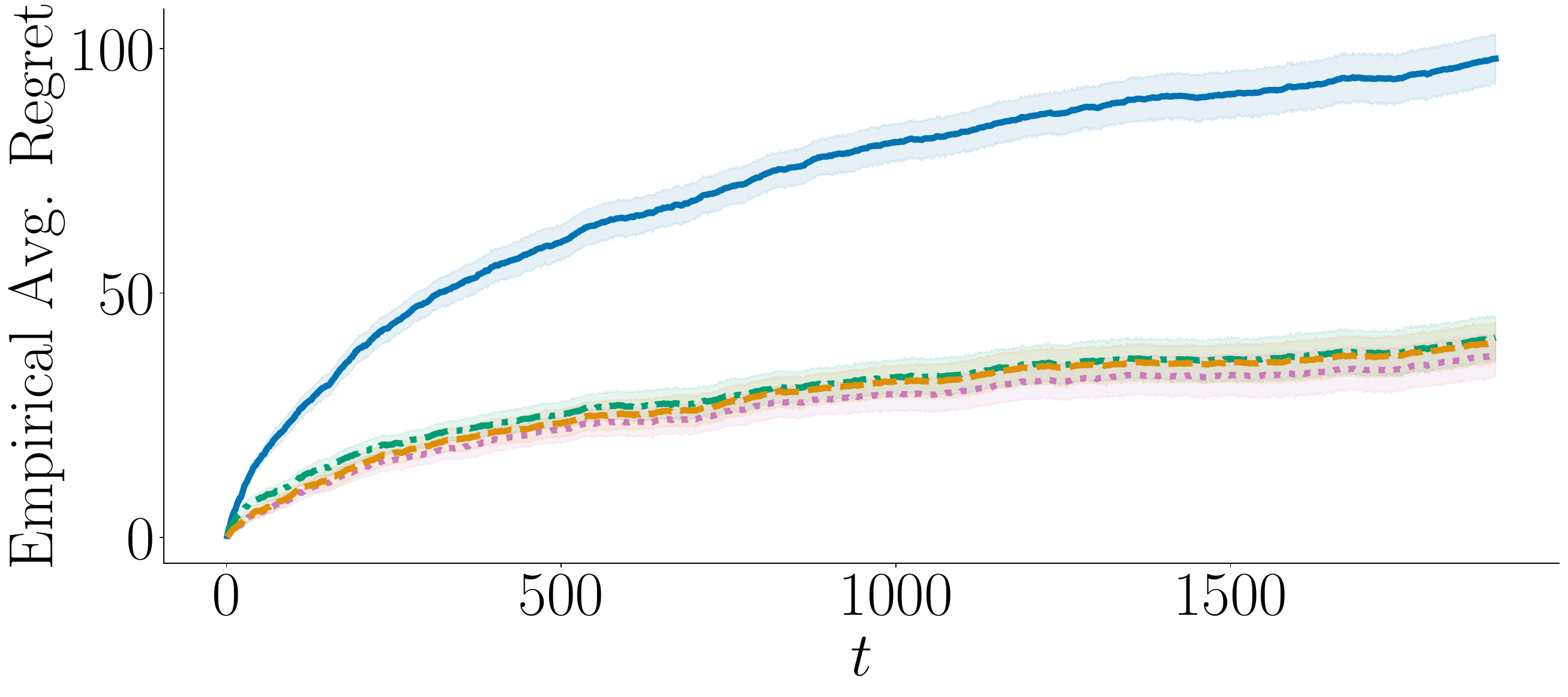}
    } \hspace{5mm}
    \subfloat[Cities]{
    \includegraphics[width=0.38\linewidth]{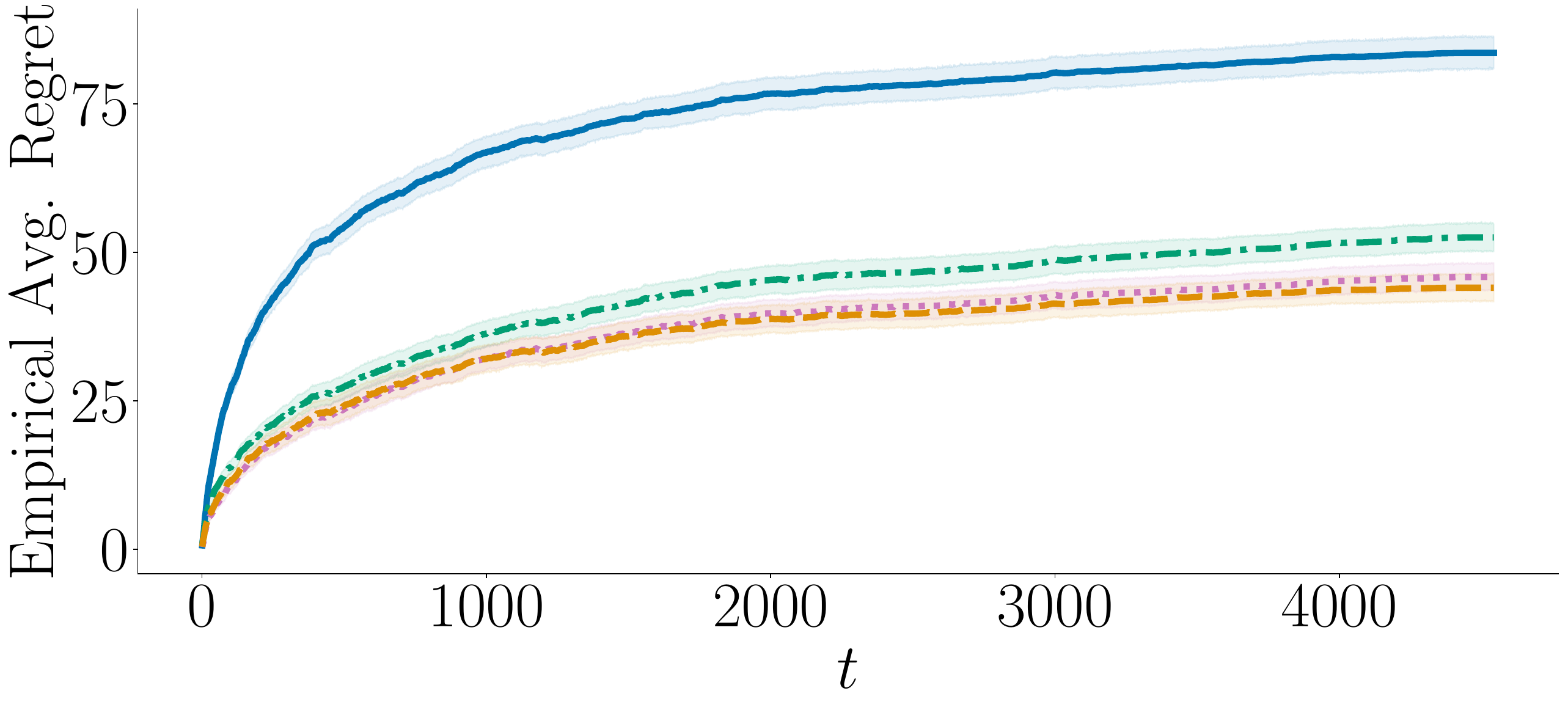}
    }\\
    
    \subfloat[Group BP]{
    \includegraphics[width=0.38\linewidth]{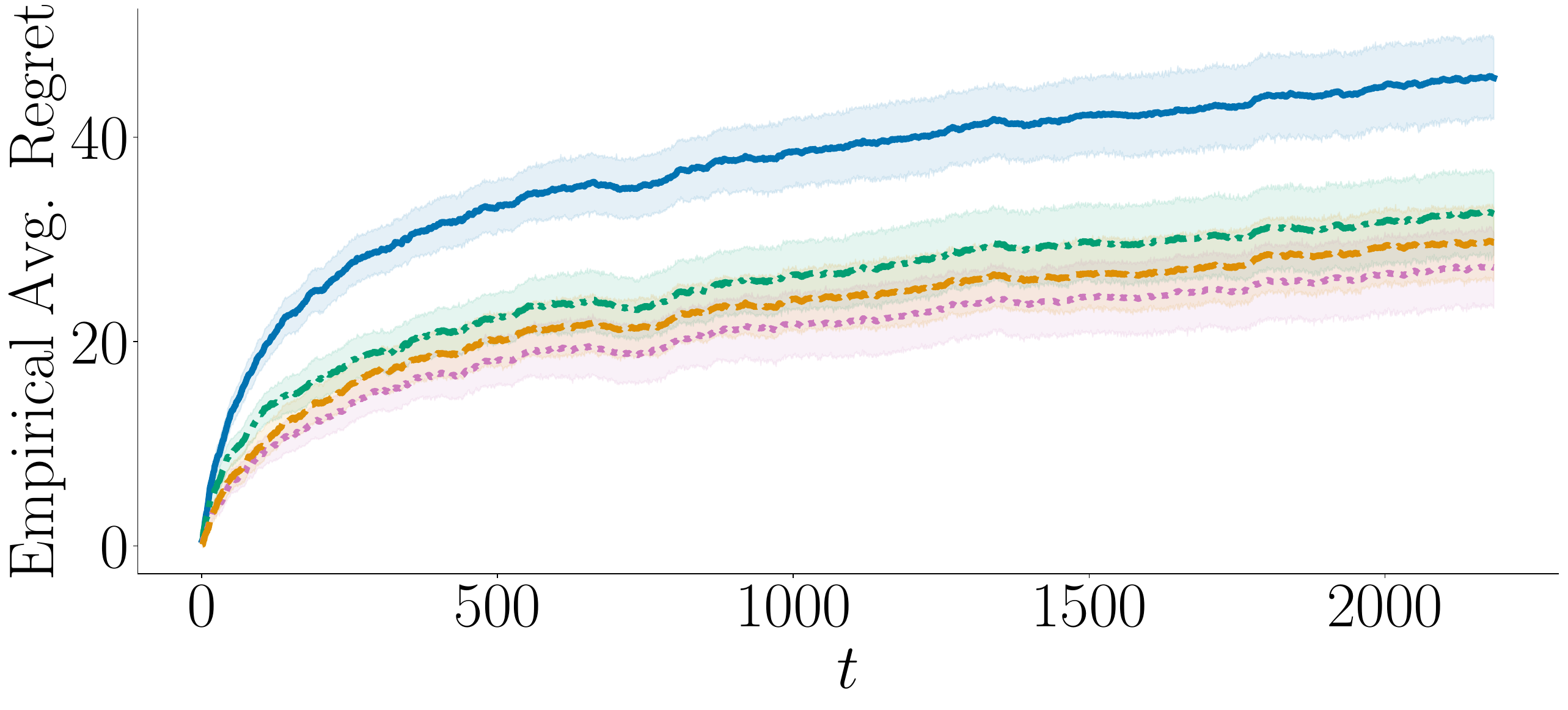}
    } \hspace{5mm}
    \subfloat[Sarcasm]{
    \includegraphics[width=0.38\linewidth]{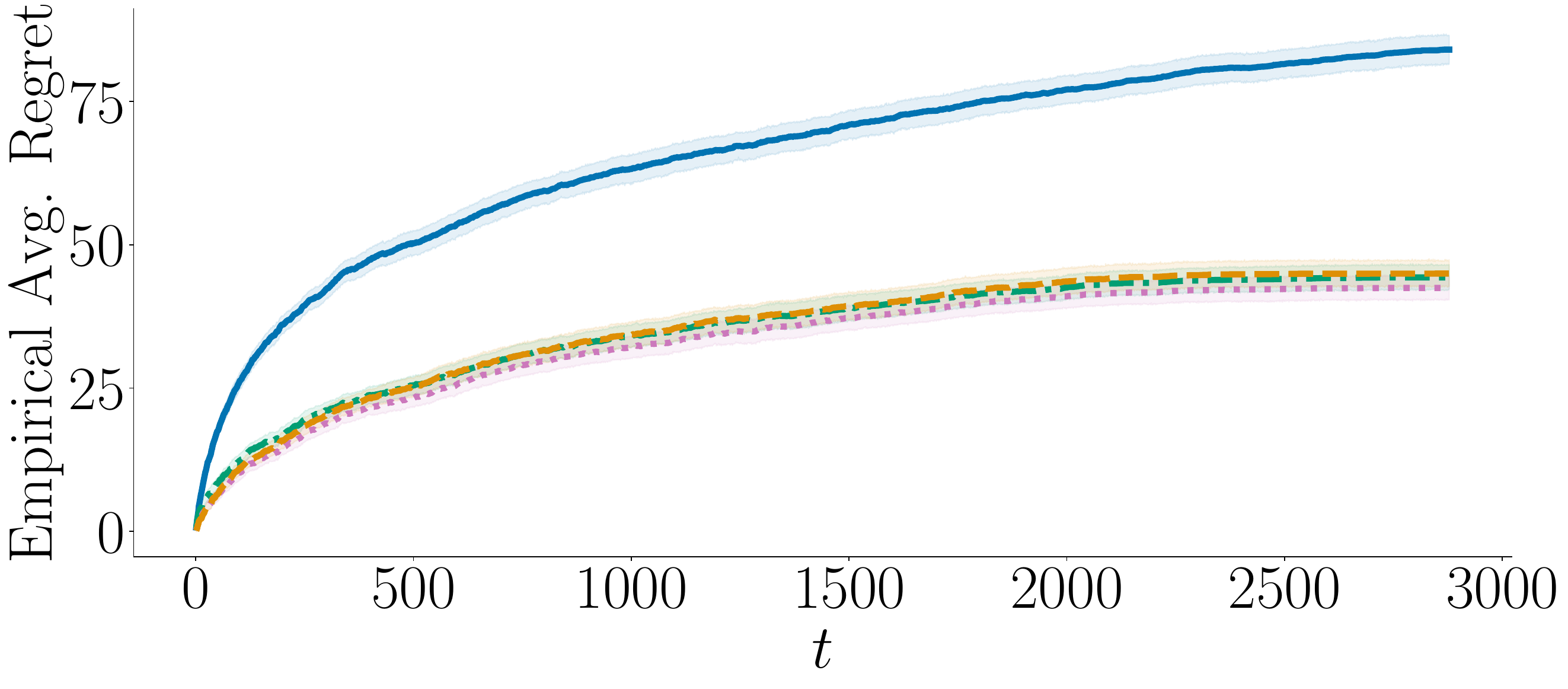}
    }
    \caption{Empirical average regret attained by our algorithm (Algorithm~\ref{alg:bandit-alg}), online isotonic regression, follow-the-leader, and a vanilla contextual online learning baseline (Algorithm~\ref{alg:contextual-alg}) per time step $t$ for each group in the Human-Alignment dataset (left) and each task in the Human-AI Interactions dataset (right).
    The average is over $100$ executions and shaded areas represent $95\%$ confidence intervals.}
    \label{fig:regret}
    \vspace{-4mm}
\end{figure}
In this section, we estimate the expected regret attained by our algorithm (Algorithm~\ref{alg:bandit-alg}), three classical online learning algorithms that can exploit the monotonicity of the outcome probabilities under perfect alignment, and a vanilla contextual online learning baseline (Algorithm~\ref{alg:contextual-alg}) using real data from two human-subject studies. In these studies, participants solve simple decision-making tasks assisted by AI models with different degrees of alignment between AI and human confidence, as measured by the maximum alignment error (MAE) and the expected alignment error (EAE)\footnote{The MAE and EAE are defined as $\text{MAE} = \max_{h\leq h', b\leq b'} P(Y=1 \given H=h, B=b) - P(Y=1 \given H=h', B=b')$ and $\text{EAE} = \frac{1}{|\Hcal|\cdot|\Bcal|}\sum_{h\leq h', b\leq b'} [P(Y=1 \given H=h, B=b) - P(Y=1 \given H=h', B=b')]_{+}$}.

\begin{figure}[ht!]
    \centering
    \includegraphics[width=.55\linewidth]{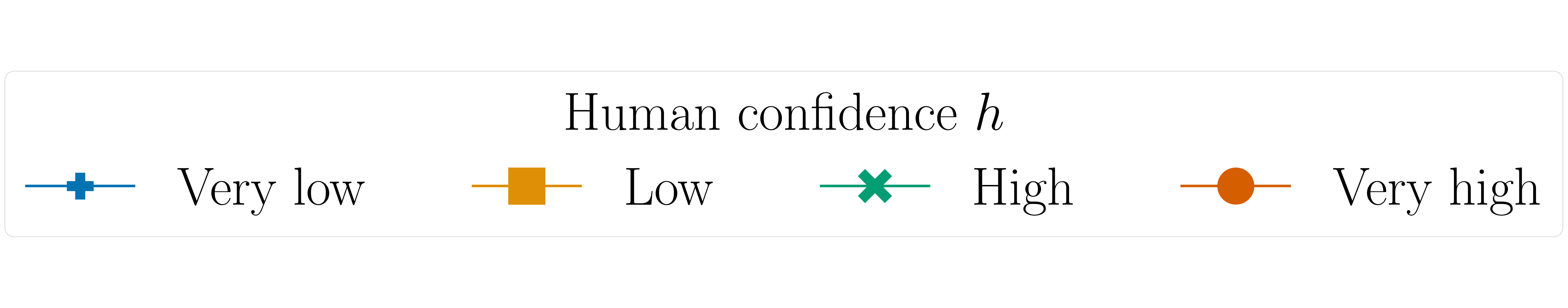}\\
    \subfloat[Group A]{
    \includegraphics[width=0.38\linewidth]{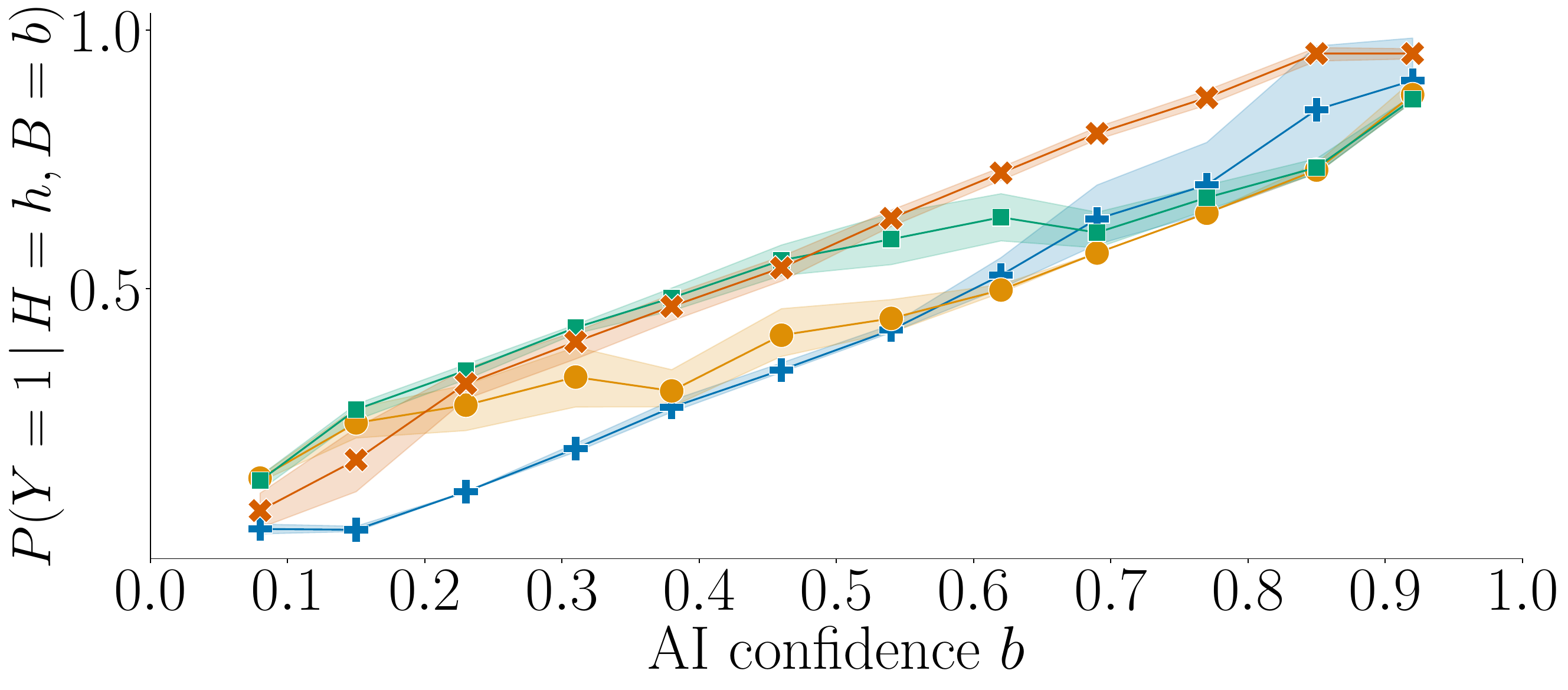}
    } \hspace{5mm}
    \subfloat[Art]{
    \includegraphics[width=0.38\linewidth]{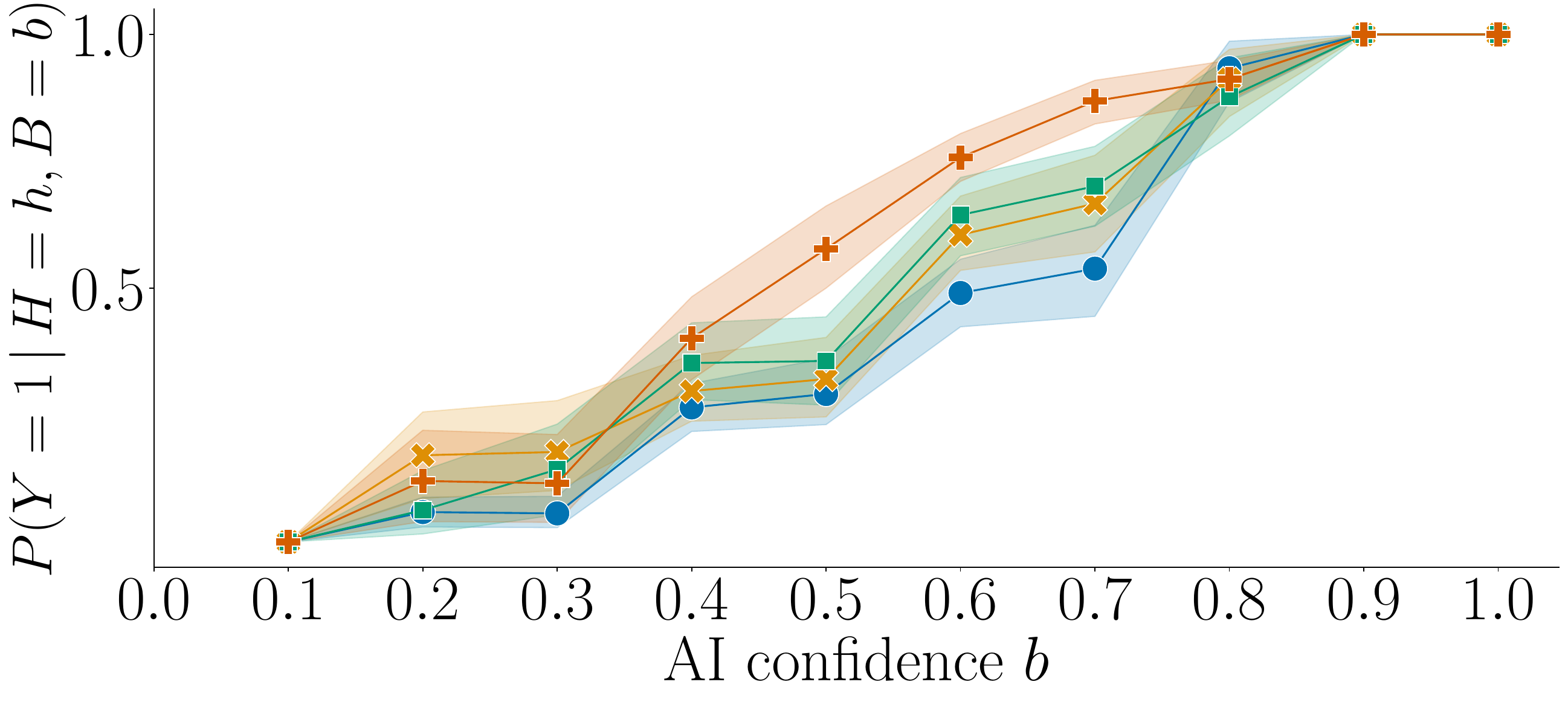}
    }\\
    
    \subfloat[Group B]{
    \includegraphics[width=0.38\linewidth]{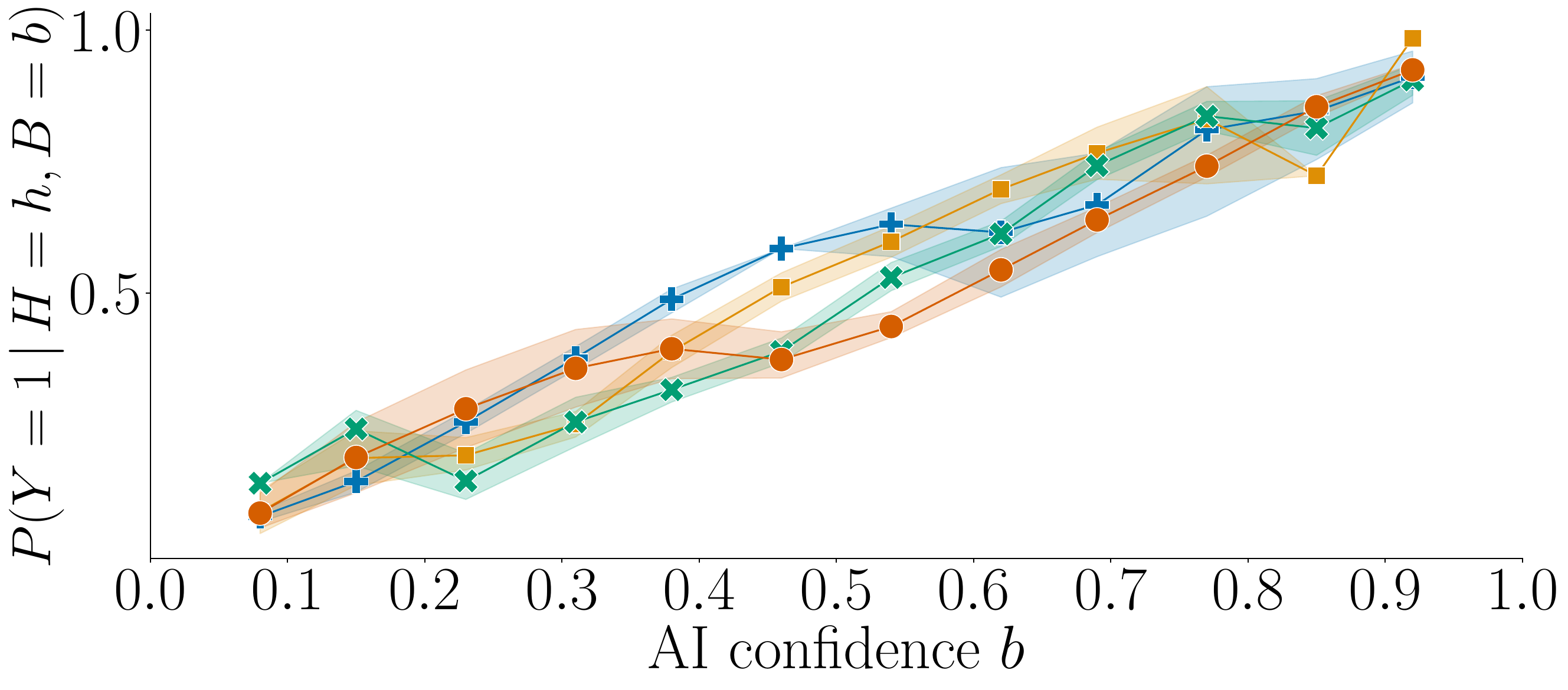}
    } \hspace{5mm}
    \subfloat[Census]{
    \includegraphics[width=0.38\linewidth]{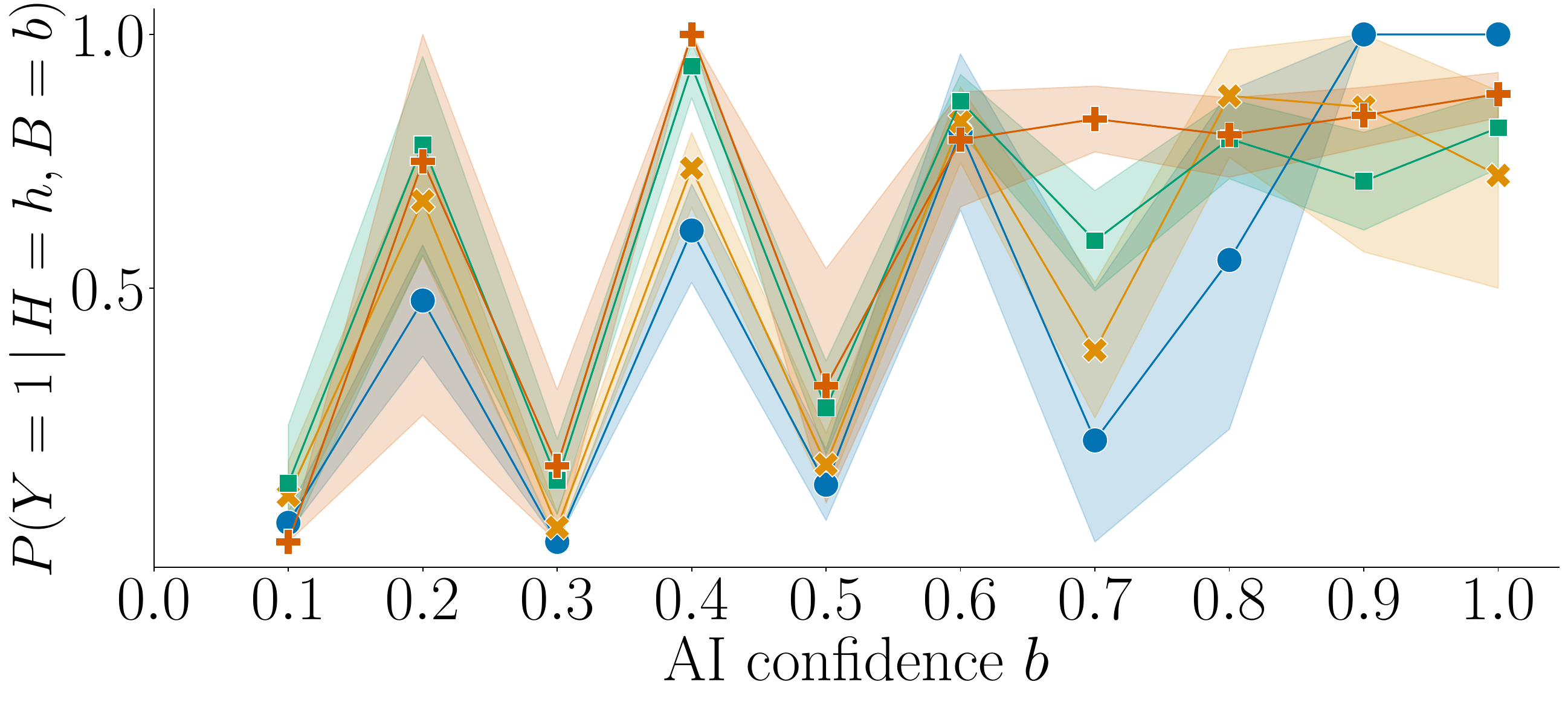}
    }\\
    
    \subfloat[Group C]{
    \includegraphics[width=0.38\linewidth]{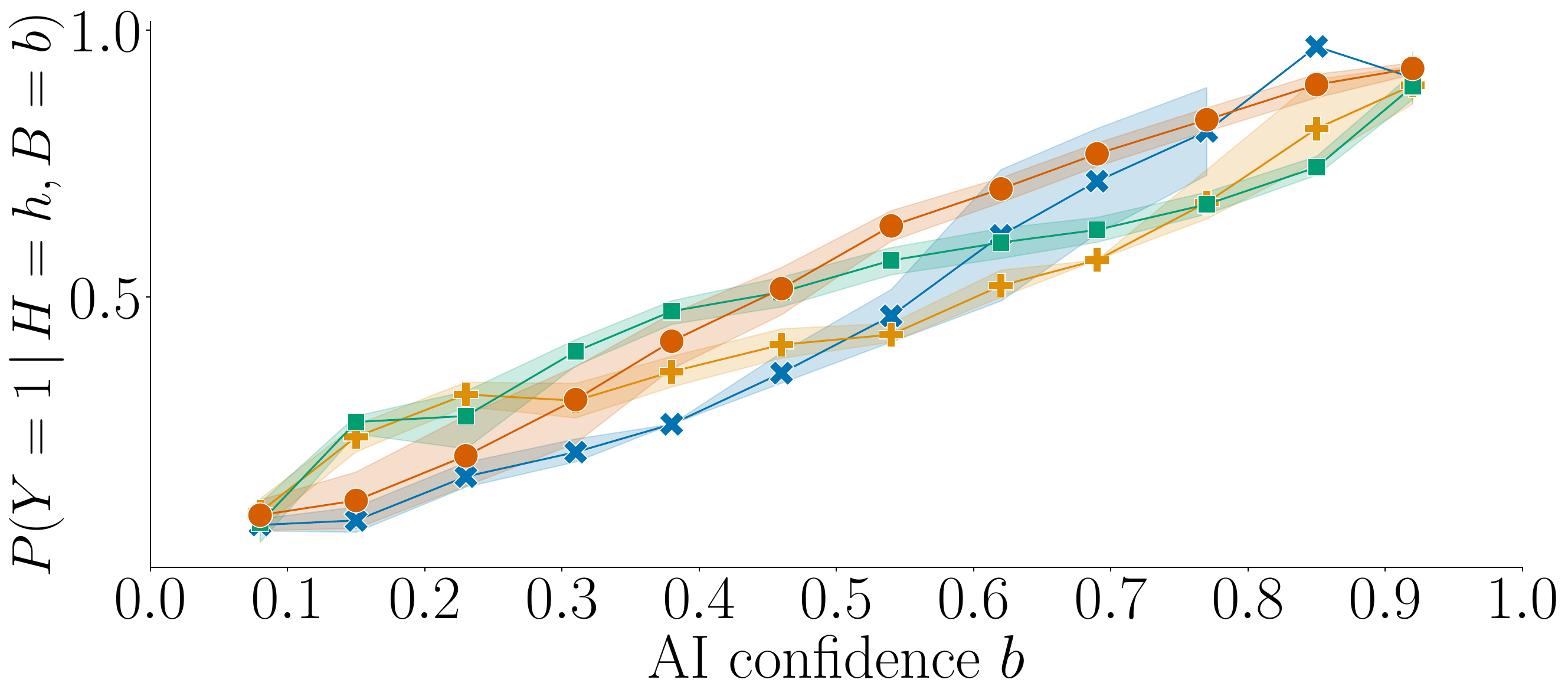}
    } \hspace{5mm}
    \subfloat[Cities]{
    \includegraphics[width=0.38\linewidth]{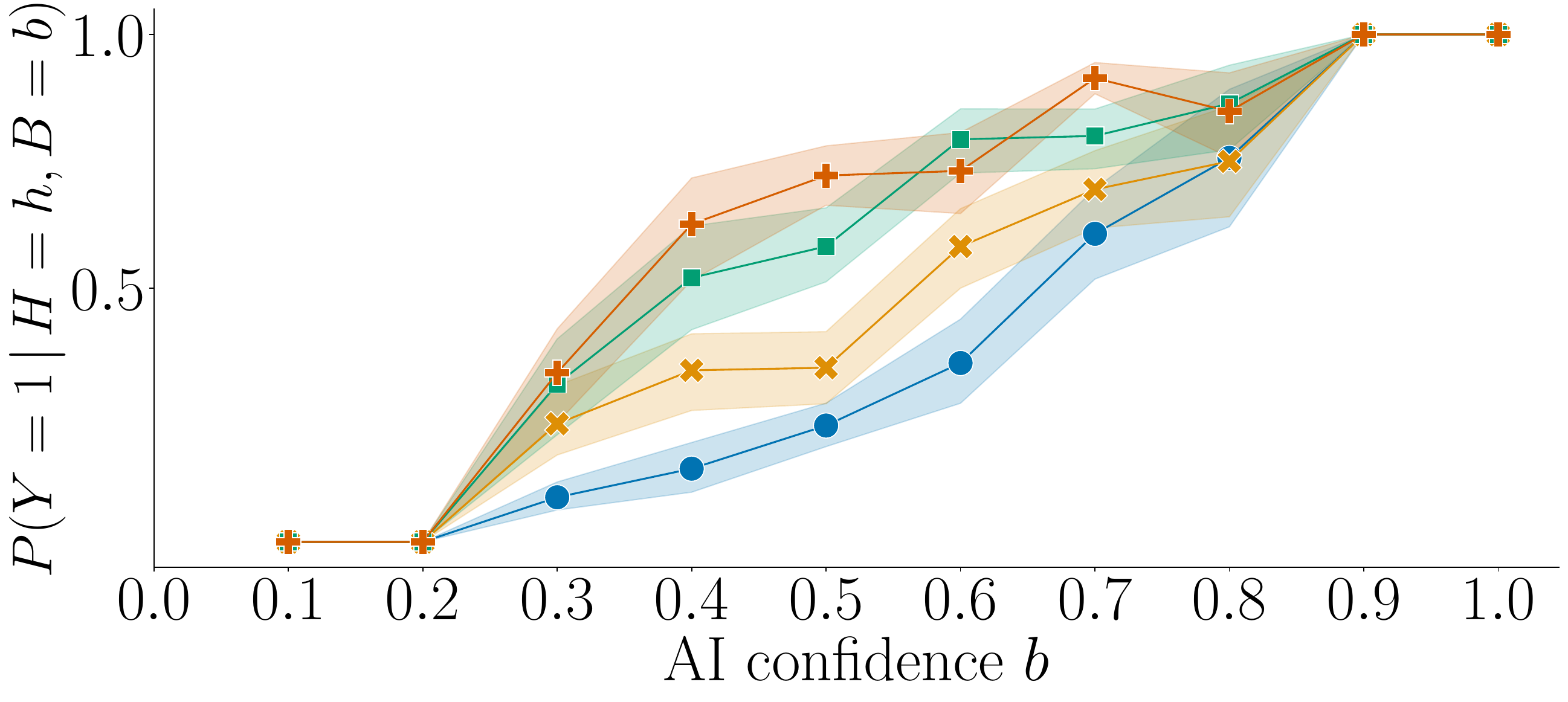}
    }\\
    
    \subfloat[Group BP]{
    \includegraphics[width=0.38\linewidth]{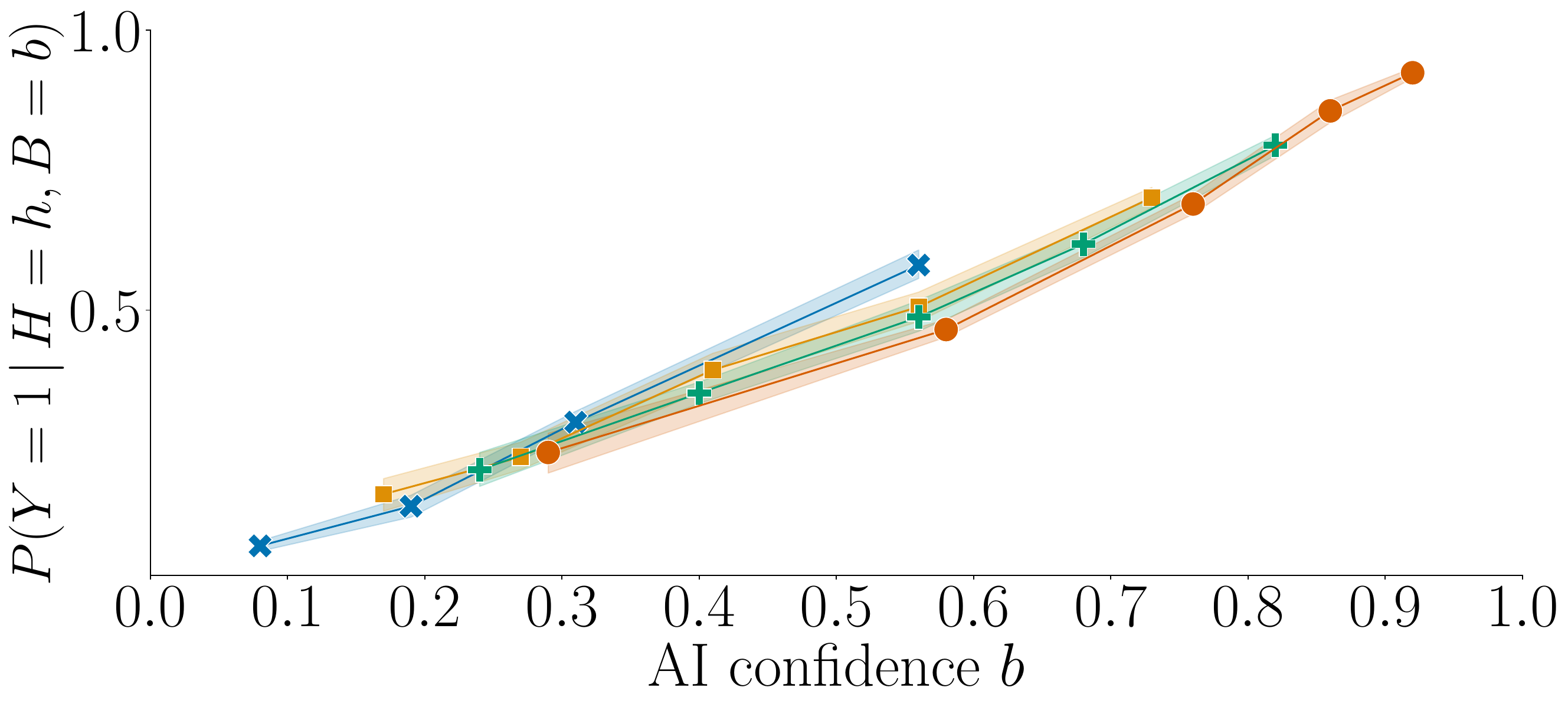}
    } \hspace{5mm}
    \subfloat[Sarcasm]{
    \includegraphics[width=0.38\linewidth]{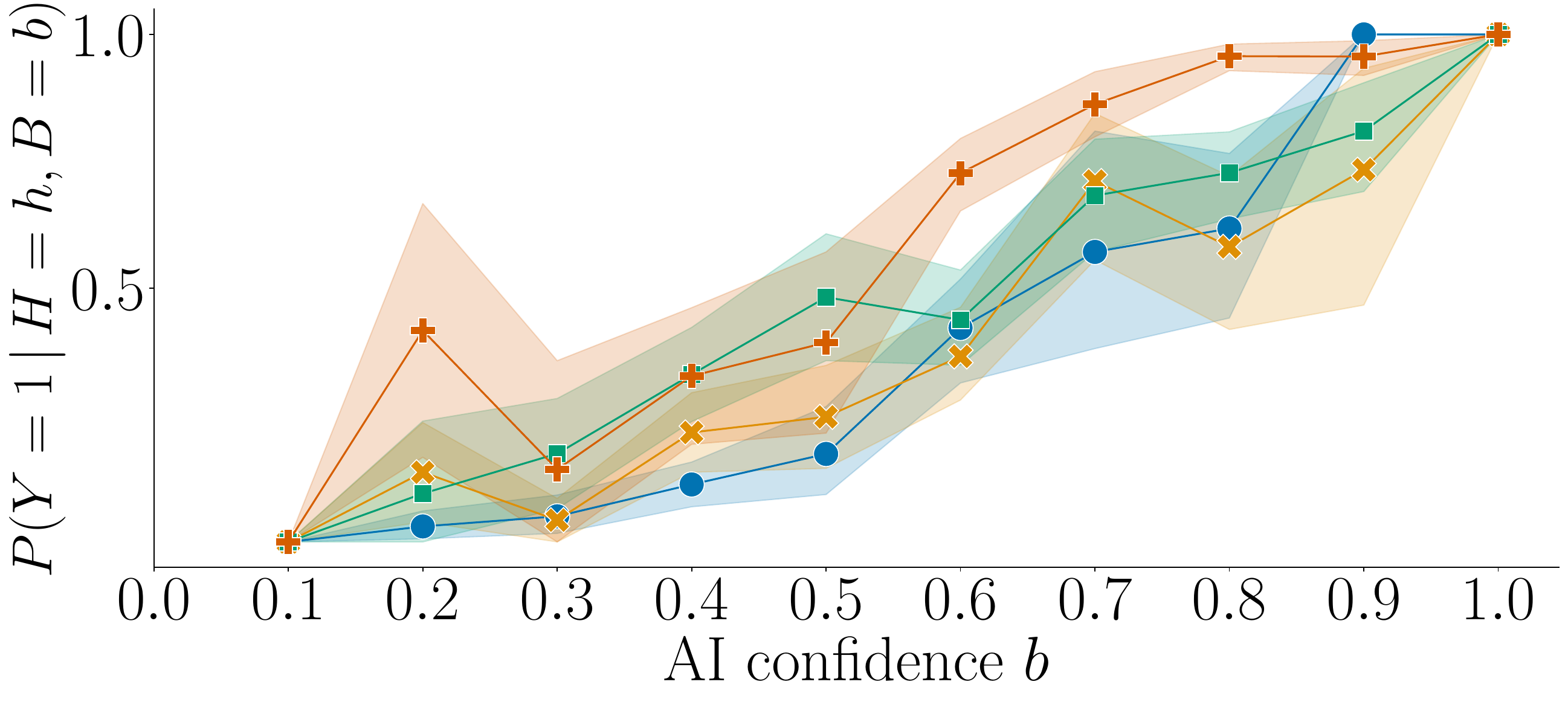}
    }
    \caption{Empirical probability $P(Y=1 \given H=h, B=b)$ for each value of the human confidence $h$ and AI confidence $b$ for each group in the the Human-Alignment dataset (left) in the Human-AI Interactions dataset (right).
    Each line corresponds to the empirical probability value for fixed $h$ against values of $b$.
    Shaded areas represent $95\%$ confidence intervals.}
    \label{fig:alignment}
    \vspace{-4mm}
\end{figure}

\xhdr{Human-Alignment dataset by~\citet{corvelo2025human}} 
The Human-Alignment dataset was gathered through a large-scale ($n=703$) human-subject study where participants predicted the binary outcome of a card game and reported their confidence both before and after receiving AI assistance.
Here, the participants' predictions and game outcomes correspond to the decisions $a$ and the label of interest $y$, respectively. 
In the study, participants are assigned to one of four groups: A, B, BP, and C.
In groups A, B, and C, participants are assisted by the same AI model; however, the distribution of cards shown to them differs, varying the degree of alignment between AI and human confidence across groups.
In group BP, participants see the same card distribution as in group B, but they are assisted by a modified version of the AI model used in the other groups, which has been postprocessed to increase alignment using held-out data from group B.
%
%
Throughout all groups, participants are asked to state their confidence in a positive prediction ($a=1$) on a scale of $[0,100]$.
Afterwards the confidence of the AI assistant in a positive prediction is shown to the participants in the same scale and participants are asked to make a final prediction.
Analogously to~\citet{corvelo2025human}, we bin participants' confidence into $|\Hcal| = 4$ confidence value classes, \{very low, low, high, very high\}, and rescale the AI confidence to $[0,1]$.
%
The game outcomes are sampled from a distribution $P(Y \given q)$,
where the value $q$ is a game-specific parameter unknown to participants (refer to Appendix~\ref{app:experiments} for details).
Further, in groups A, B, and C, the AI confidence is independent of the human confidence and $|\Bcal| = 13$. In contrast, in group BP, the AI confidence depends on the human confidence and, for each $h$, $|\Bcal| = 5$.
Table~\ref{tab:human-alignment-dataset} in Appendix~\ref{app:experiments} shows the MAE and EAE, and the number of data samples per group.

\xhdr{Human-AI Interactions dataset by~\citet{vodrahalli2022humans}} 
The Human-AI Interactions dataset was gathered through a large-scale human subject study where participants solved different binary classification tasks with AI assistance. 
Throughout our experiments, we used a subset of this dataset  processed by~\citet{corvelo2023human}, which comprises $15{,}063$ unique binary predictions from $471$ participants, together with their confidence, on four tasks: i) `Art', ii) `Sarcasm', iii) `Cities', and iv) `Census'. 
Participants had to make a binary prediction on i) the art period of a painting (given two options), ii) presence of sarcasm in text (yes/no option), iii) which US city is shown in an image (given two options), and iv) if an individual earns more than $50$K a year given their census data (yes/no option), respectively.    
Participants' confidence is recorded on a scale of $[-1,1]$, where $-1$ and $1$ represent full confidence in one of the options respectively, and $0$ represents no confidence in either option. After the initial confidence elicitation\footnote{In the experiments, we take the reported confidence as ground truth for the decision maker's confidence. However, noisy reporting of the decision maker's confidence can have an effect on the alignment of the AI which we discuss in section~\ref{sec:discussion}.}, the AI confidence is shown to participants on the same scale as the one used to record the human confidence, and participants can adjust their confidence and final prediction.
Throughout all tasks, the participants' predictions and the groundtruth correspond to the decisions $a$ and the label of interest $y$, respectively, where we define the positive outcome ($y=1$) and decision ($a=1$) at random to one of the two given options for tasks i) and iii) and to the "yes" option for task ii) and iv).
In each task, the AI model used to provide assistance is fitted on labels by human annotators, the AI confidence is a noisy average of the annotators'{} label confidence, and the degree of alignment varies across tasks.
For better comparison in our experiments, we bin participants' confidence into $|\Hcal| = 4$ confidence value classes, \{very low, low, high, very high\}, and the AI confidence to $|\Bcal|=10$ confidence value classes rescaled to $[0,1]$.
Table~\ref{tab:haii-dataset} in Appendix~\ref{app:experiments} shows the MAE and EAE, and the number of data samples per task.

\xhdr{Experimental Setup} 
We estimate the expected regret attained by our algorithm (Algorithm~\ref{alg:bandit-alg}), two classical online learning algorithms that can exploit the monotonicity of the outcome probabilities under perfect alignment, and a vanilla contextual online learning baseline (Algorithm~\ref{alg:contextual-alg}) on each of the groups (A, B, BP and C) and tasks (`Art', `Sarcasm', `Cities' and `Census'), using all the data samples. 
%
The two classical algorithms are: i) online isotonic regression~\citep{kotlowski2016online}, ii) follow-the-leader algorithm (FTL)~\citep{kotlowski2018minimaxity}. 
In both algorithms, one instance of the algorithm is used for each $h \in \Hcal$ observed and, for FTL, we define the experts as threshold classifiers over $\Bcal$ and $\Bcal$ needs to be known a priori. 
%
See Appendix~\ref{app:experiments} for detailed descriptions of each algorithm.
In all our experiments, we use the utility function $u(a, y) = \II[a = y] - \II[a \neq y], \forall a,y \in \{0,1\}$ and, to estimate the expected regret, we repeat each experiment $100$ times using different random seeds and, in each repetition, we randomly shuffle all data samples.

\xhdr{Results} Figure~\ref{fig:regret} shows the (estimated) expected regret attained by our algorithm, online isotonic regression, FTL, and the baseline on both datasets.
%
We find that, with the exception of the `Census' task, our algorithm, the online isotonic regression algorithm and the FTL algorithm learn the optimal policy for all groups and tasks and achieve comparable regret, lower than the vanilla baseline, even though perfect alignment does not hold. 
This is because the regret guarantees of the implementation of our algorithm, online isotonic regression and FTL only require the probability $P(Y = 1 \given H, B)$ to be monotone in $B$ for any given $H$---a condition representing a weaker notion of perfect alignment.
In our experiments, this probability is approximately monotone in $B$ for most $H$ values across nearly all groups and tasks, with the exception of the `Census' task as seen in Figure~\ref{fig:alignment}.
%
These findings strongly support our main claim—a high degree of alignment speeds up learning under AI-assistance, regardless of the specific learning algorithm.

\section{Discussion and Limitations}
\label{sec:discussion}
In this section, we highlight several limitations of our work, discuss its broader impact, and point out avenues for future research.

\xhdr{Monotonicity of the optimal decision policy}
In our work, we have shown that, under perfect alignment, a decision-maker can learn the optimal decision policy more efficiently by focusing only on the set of decision policies monotone in the AI confidence.
However,~\cite{corvelo2023human} have shown that, under perfect alignment, the optimal decision policy lies within the set of policies monotone in both the human and the AI confidence, which is a subset of the set of decision policies monotone in the AI confidence.
Therefore, it would be very interesting to investigate whether, under perfect alignment, a decision-maker can learn the optimal decision policy even more efficiently by focusing on this latter set.

\xhdr{Decision making setting}
We have considered AI-assisted decision making processes where the label $y$ is always observable independently of the decision taken---there is full-feedback. However, there exist decision making tasks in which the label $y$ is only realized whenever the decision $a = 1$. For example, we can only observe if a treatment is eventually beneficial or not if the patient is treated, or we can only observe if a customer does or does not default on a loan if the loan is granted. %
Our lower bound on the expected regret can be easily extended to such partial-feedback setting using a classical result for contextual bandits by~\citet{agarwal2012contextual}.
However, investigating the extent to which a decision-maker can leverage the alignment between AI and human confidence to learn more efficiently under partial-feedback is significantly more challenging, and it is left for future work. 
%
%
%
Further, in our setting, we have assumed that the AI model and the degree of alignment are fixed. 
It would be interesting to extend our analysis to a setting where both the AI model and the degree of alignment can change over time. To this end, recent work on online calibration~\citep{gupta2023online} and online multicalibration~\citep{garg2024oracle,ghuge2025improved} may be proven useful.

\xhdr{Broader impact}
AI promises to enhance decision making in a variety of high-stakes domains; however, empirical evidence suggests that decision-makers often struggle to make optimal decisions under AI assistance. 
Through theoretical and empirical analysis, our work suggests that increasing the alignment between the AI confidence and the decision maker's confidence on their own predictions may help decision-makers learn to make optimal decisions through repeated interactions. In practice, we can test for alignment by gathering confidence values from a decision maker and the corresponding outcomes from their decisions. Using this data, one can i) measure the degree of alignment of different AI models, and ii) improve alignment of any model through post-processing using multicalibration algorithms~\citep{corvelo2023human}. Models with high degree of alignment and good performance can then be used to assist the human decision maker. Note that problems with noisy reporting of confidence values by human decision makers can affect the eventual alignment checks or post-processing of the AI-assistant for alignment. However, we have empirically shown that our results are robust to violations of perfect alignment.

\section{Conclusions}
\label{sec:conclusions}
We have theoretically shown that, when AI models assist decision-makers by predicting a binary outcome of interest, perfect alignment between the AI confidence and the decision-makers' confidence in their own predictions can be used to learn the optimal decision policy more efficiently.
Moreover, using real data from two human-subject studies, we have empirically demonstrated that  our theoretical results are robust to certain violations of perfect alignment---a learner who assumes perfect alignment can achieve lower expected regret than one making no such assumption, provided that a weaker notion of alignment holds.

\vspace{2mm}
\xhdr{Acknowledgements} 
Gomez-Rodriguez acknowledges support from the European Research Council (ERC) under the European Union'{}s Horizon 2020 research and innovation programme (grant agreement No. 101169607).
Straitouri acknowledges support from a Google PhD Fellowship.

{ 
\small
\bibliographystyle{unsrtnat}
\bibliography{online-human-alignment}
}

\clearpage
\newpage

\appendix

\section{Proofs}
\label{app:proofs}

\subsection{Proof of Theorem~\ref{th:vanilla-lower-bound}}
Our proof closely follows the proof of  
Theorem 5.1 in~\cite{agarwal2012contextual}, which shows that learning the optimal policy $\pi^*$ for the contextual multi-arm bandit problem with a policy class $\Pi$ and $K$ arms (refer to~\cite{slivkins2019introduction} Chapter 8.4 for the exact definition) has expected regret $\EE(R(T)) = \Omega \left ( \sqrt{K \cdot T \cdot \log(|\Pi|) / \log(K)}\right)$, as long as $\pi^{*} \in \Pi$.
This is because learning the best-response policy $\pi^*(h,b)$ given by Eq.~\ref{eq:best-response-policy} is equivalent to an instance of the contextual bandit problem with a policy class with context $z = (h,b)$, $K=2$ arms with each decision $a$ being an arm, and a policy class $\Pi = \{0,1\}^{|\Hcal| \cdot |\Bcal|}$.
However, the key difference between our setup and the one considered in~\cite{agarwal2012contextual} is that at each time step, rather than a single utility, we observe the utility for all arms, \ie, we assume \textit{full feedback}.

Under this assumption we show that we achieve a tighter lower bound for the expected regret, which is essentially due to achieving a tighter upper bound in Lemma A.1 by~\cite{auer2002nonstochastic}. 
This Lemma considers a $K$-arm bandit problem with a horizon of $T$ time steps and the following utility distributions: a) the distribution $P_i$, where the utility $R\sim Bernoulli(1/2 +\epsilon)$, $\epsilon \in (0,1/2)$ for an arm $i$ chosen uniformly at random before the first time step and for any arm $i' \neq i$ the utility $R\sim Bernoulli(1/2)$; b) the distribution $P_{unif}$, where the utility $R\sim Bernoulli(1/2)$ for all arms. 
In what follows, we will first derive the equivalent of Lemma A.1 by~\cite{auer2002nonstochastic} under full feedback, on which we will base the rest of our proof. 
\begin{lemma}~\label{lem:auer-a-1}
    Let $f$ be any function $\{0,1\}^{K \times T} \rightarrow [0, M]$ defined on a utility matrix $\mathbf{R}$ comprising the utilities under full feedback for any bandit algorithm over $T$ time steps. Then for any decision $i$,
    \begin{align*}
        \EE_{i}[f(\mathbf{R})]  - \EE_{unif}[f(\mathbf{R})] \leq \frac{M}{2}\sqrt{T \cdot \frac{-\log(1 - 4\epsilon^2)}{4}},
    \end{align*}
    where $\EE_{i}$, $\EE_{unif}$ denote the expectation with respect to the distribution $P_i$ and $P_{unif}$ respectively.  
\end{lemma}
\begin{proof}
    Similar to Lemma A.1~\cite{auer2002nonstochastic} we have
    \begin{align}~\label{eq:total-variaton}
        \EE_{i}[f(\mathbf{R})]  - \EE_{unif}[f(\mathbf{R})] &= \sum_{\mathbf{R}}f(\mathbf{R})(P_{i}[\mathbf{R}] - P_{unif}[\mathbf{R}]) \leq M \cdot \sum_{\mathbf{R}}(P_{i}[\mathbf{R}] - P_{unif}[\mathbf{R}]) \nonumber \\
        &\leq M \cdot \sum_{\mathbf{R}: P_{i}[\mathbf{R}]  \geq P_{unif}[\mathbf{R}]}(P_{i}[\mathbf{R}] - P_{unif}[\mathbf{R}])
        = \frac{M}{2} || P_i - P_{unif}||_1,
    \end{align}
    where $|| P_i - P_{unif}||_1$ is the variational distance  between the distributions $P_i$ and $P_{unif}$, defined as 
    \begin{align*}
        || P_i - P_{unif}||_1 = \sum_{\mathbf{R} \in \{0,1\}^{K\times T}} \left | P_{i}(\mathbf{R}) - P_{unif}(\mathbf{R}) \right |
    \end{align*}
    the last equality in Eq.~\ref{eq:total-variaton} holds because $P_{i}[\mathbf{R}] \geq P_{unif}[\mathbf{R}]$ whenever the realized utility values in $\mathbf{R} \in \{0,1\}^{T\times K}$ are $1$ for the optimal arm $i$ ---since in this case these values are sampled with probability $1/2 + \epsilon$ for arm $i$ and with $1/2$ for all other arms.

By the Pinsker'{}s inequality it holds 
    \begin{align}~\label{eq:pinsker}
        || P_{unif} - P_i||_1 \leq \sqrt{\frac{1}{2}KL(P_{unif} || P_i)},
    \end{align}
where $KL(P_{unif} || P_i)$ is KL divergence between $P_i$ and $P_{unif}$.
    By the chain rule of total entropy we have 
    \begin{align}~\label{eq:kl-divergence}
        KL(P_{unif}|| P_i ) = \sum_{t=1}^{T}KL(P_{unif}(\mathbf{r_t}) || P_i(\mathbf{r_t})),
    \end{align}
    where $\mathbf{r_t} \in \{0,1\}^K$ is a $K$ column vector of the utility matrix $\mathbf{R}$ at column $t$, and we use that $\mathbf{r_t}$ is independent of $\mathbf{r_{t-1}}$ for $t\geq 2$. We will now use the expression of KL divergence for multivariate Bernoulli distributions $\mathbf{p},\mathbf{q}$ with $K$ variables which is as follows
    \begin{align*}
        KL(\mathbf{p} || \mathbf{q}) = \sum_{j=1}^{K}p_{j} \log\frac{p_{j}}{q_{j}} + (1 - p_{j} )\log\frac{1 - p_{j}}{1 - q_{j}},
    \end{align*}
    where $p_j, q_j$ is the $j$-th coefficient of $\mathbf{p},\mathbf{q}$ respectively. Since $P_i(\mathbf{r_t}), P_{unif}(\mathbf{r_t})$ are multivariate Bernoulli distributions
    for any utility vector $\mathbf{r}_{t} = \mathbf{r}$ it holds 
    \begin{align*}
         KL(P_{unif}(\mathbf{r}) || P_i(\mathbf{r})) &= \sum_{j=1}^{K}P_{unif}(r_j)\log\frac{P_{unif}(r_j)}{P_i(r_j) } + (1 -  P_{unif}(r_j))\log\frac{1 - P_{unif}(r_j)}{1 - P_i(r_j)}\\
         &= P_{unif}(r_i) \log\frac{P_{unif}(r_i)}{P_i(r_i)} + (1 - P_{unif}(r_i)\log\frac{1 - P_{unif}(r_i)}{1 - P_i(r_i)} \\
         &= \frac{1}{2}\log\frac{1/2}{1/2 + \epsilon} + \frac{1}{2}\log\frac{1/2}{1/2 -\epsilon}\\
         &= \frac{1}{2}\log(\frac{1}{1 + 2\epsilon}\cdot \frac{1}{1 - 2\epsilon})\\
         &= \frac{1}{2}\log(\frac{1}{1 + 2\epsilon}\cdot \frac{1}{1 - 2\epsilon})\\
         &= \frac{1}{2}\log\frac{1}{1 - 4\epsilon^2} \\
         &= -\frac{1}{2}\log(1 - 4\epsilon^2),
    \end{align*}
    where we use that for any arm $j\neq i, P_{unif}(r_j)=P_{i}(r_j) = 1/2$. Therefore Eq.~\ref{eq:kl-divergence} becomes
    \begin{align}~\label{eq:kl-divergence-bound}
          KL(P_{unif}|| P_i ) = \sum_{t=1}^{T}KL(P_{unif}(\mathbf{r_t}) || P_i(\mathbf{r_t})) = -\frac{T}{2}\log(1 - 4\epsilon^2)
    \end{align}
    By Eq.~\ref{eq:total-variaton},~\ref{eq:pinsker}, and~\ref{eq:kl-divergence-bound} we have 
    \begin{align*}
        \EE_{i}[f(\mathbf{R})]  - \EE_{unif}[f(\mathbf{R})] \leq \frac{M}{2}\sqrt{T \cdot \frac{-\log(1 - 4\epsilon^2)}{4}}
    \end{align*}
\end{proof}
Note that the above bound no longer depends on the expected number of times that the algorithm chose to play the arm $i$ as in~\cite{auer2002nonstochastic}. 

We will use the above result to prove Corollary~\ref{cor:cor-5.1-agarwal} below, in the context of the construction used by~\cite{agarwal2012contextual}, which we restate for completeness: We assume $M$ different contexts, a policy class $\Pi$ comprising all possible $K^{M}$ policies over the contexts, \ie, $M = \log(|\Pi|)/\log(K)$, and the utility distributions as in the above lemma for each context $z$. Fix a context $z$, let $T_{z}$ be the set of time steps such that $z_t=z$, and let $i$ be the optimal arm for this context under the utility distribution $P_{i}$ chosen uniformly at random before the first time step. Then by Lemma~\ref{lem:auer-a-1}, for the number of time steps in $T_{z}$ such that any algorithm chooses $a_t = a$ we can show:

\begin{corollary}~\label{cor:cor-5.1-agarwal}
    Given context $z_t=z$ for $t=1,\dots,T$ 
    \begin{align*}
        \EE_{a}[S_{a} ] \leq \EE_{unif}[S_{a}] + |T_{z}|^{\frac{3}{2}} \cdot \sqrt{2}\epsilon.
    \end{align*}
\end{corollary}
\begin{proof}
    We use Lemma~\ref{lem:auer-a-1} for  $f(\mathbf{R}) = S_{a} \in [0, |T_{z}|]$ (because given the context $z$, $S_{a} \leq |T_{z}|$ ),  horizon $T=|T_z|$, therefore we have
    \begin{align*}
        \EE_{a}[S_{a}] - \EE_{unif}[S_{a}] \leq \frac{|T_{z}|}{2} \sqrt{|T_{z}|} \cdot \frac{-\log(1 - 4\epsilon^2)}{4} \leq |T_{z}|^{\frac{3}{2}} \sqrt{\frac{-\log(1 - 4\epsilon^2)}{16}}.
    \end{align*}
    We will now use that $(-\log(1 - 4\epsilon^2))/16 \leq 2\epsilon^2$, which we can rewrite as $0 \leq 8w - \log\left(1 - w \right)$, for $w = 4\epsilon^2$, $w \in (0,1)$. Let $g(w) = 8w - \log\left(1 - w \right)$. Since, $g'(w) = 8 + \frac{1}{1-w} > 0$ for $w \in (0,1)$, $g$ is increasing for $w \in (0,1)$, and thus $g(w) \geq \lim_{w \rightarrow 0}g(w) = \lim_{w \rightarrow 0} 8w - \log\left(1 - w \right)  = 0$.
\end{proof}

By using Corollary~\ref{cor:cor-5.1-agarwal} instead of Corollary 5.1 in the rest of the proof by~\cite{agarwal2012contextual} we have that the expected regret 
\begin{align*}
    \EE[R(T)] = \Omega\left( M \cdot \epsilon - \frac{\epsilon^2}{\sqrt{M}} \cdot T^{3/2}\right),
\end{align*}
which for $\epsilon = \Theta(\sqrt{M/T})$ becomes 
\begin{align*}
     \EE[R(T)] = \Omega\left(\sqrt{T \cdot M}\right) = \Omega\left(\sqrt{T \cdot \frac{\log(|\Pi|)}{\log{K}}}\right).
\end{align*}
Since $|\Pi| = K^{|\Hcal| \cdot |\Bcal|}$ we have
\begin{align*}
    \EE[R(T)] = \Omega \left ( \sqrt{ T \cdot |\Hcal| \cdot |\Bcal|} \right)
\end{align*}%


\subsection{Proof of Theorem~\ref{th:optimal-threshold}}

Let $h\in\Hcal$. Due to perfect alignment, for any $b,b'$ it holds that
    \[ b\leq b' \iff P(Y=0 \given H=h, B=b ) \geq P(Y=0 \given H=h, B=b' )\]
    Using this and by definition of $b^*(h)$, we have that
    \[  b\leq b^*(h) \iff P(Y=0 \given H=h, B=b ) \geq \frac{u(1,1) - u(0,1)}{u(1,1) - u(1,0) + u(0,0) - u(0,1)}.\]
By Lemma~\ref{lem:optimal-action-fixed-context}, stated and proven below, it must hold that \[ b\leq b^*(h) \iff \pi^*(h,b)=0. \]
    
\begin{lemma}~\label{lem:optimal-action-fixed-context}
    For any utility function satisfying Eq.~\ref{eq:u_prop} and any fixed confidence values $h,b$, let $\pi^*(h,b)$ denote the optimal decision in this context, \ie, $\pi^*(h,b):= \arg\max_{a\in \{0,1\}} \mu(a \given h,b)$. The following are equivalent statements
    \begin{enumerate}[label=\alph*)]
        \item $\pi^*(h,b)=0$
        \item $P(Y=0 \given H=h, B=b) \geq  \frac{u(1,1) - u(0,1)}{u(1,1) - u(1,0) + u(0,0) - u(0,1)}$
    \end{enumerate}
\end{lemma}
\begin{proof}
    It holds that  
    \begin{align}~\label{eq:mu-1}
            \mu(1 \given h, b)&= E_{Y} \left [ u(1, Y) | H = h, B=b \right ] \nonumber\\
                  &= P (Y=1 | H = h, B=b ) \cdot u(1, 1) + P (Y=0 | H = h, B=b ) \cdot u(1, 0) \nonumber \\
                  &= \left(1-P (Y=0 | H = h, B=b )\right) \cdot u(1, 1) + P (Y=0 | H = h, B=b )  \cdot u(1, 0) \nonumber\\
                  &= P (Y=0 | H = h, B=b ) \cdot \left ( u(1, 0) - u(1, 1) \right ) + u(1,1),
        \end{align}
        since $Y \in \{0,1\}$ and thus $P (Y=1 | H = h, B=b ) = 1 - P (Y=0 | H = h, B=b )$. Similarly, we show:
        \begin{align}~\label{eq:mu-0}
            \mu(0 \given h, b)&= E_{Y} \left [ u(0, Y) | H = h, B=b \right ] \nonumber \\
                  &= P (Y=1 | H = h, B=b ) \cdot u(0, 1) + P (Y=0 | H = h, B=b ) \cdot u(0, 0) \nonumber \\
                  &= \left(1-P (Y=0 | H = h, B=b )\right) \cdot u(0, 1) + P (Y=0 | H = h, B=b )  \cdot u(0, 0) \nonumber\\
                  &= P (Y=0 | H = h, B=b ) \cdot \left ( u(0, 0) - u(0, 1) \right ) + u(0,1).
        \end{align}
         By definition a decision $a'$ is optimal iff $a' =\arg\max_{a\in \{0,1\}} \mu(a \given h, b)$. Thus, $\pi^*(h,b)=0$ iff $\mu(1 \given h, b) \leq \mu(0 \given h, b)$. We can rewrite this inequality using Eq.~\ref{eq:mu-1} and Eq.~\ref{eq:mu-0} as
         \begin{multline*}
             P (Y=0 | H = h, B=b ) \cdot \left ( u(1, 0) - u(1, 1) \right ) + u(1,1) \\ \leq P (Y=0 | H = h, B=b ) \cdot \left ( u(0, 0) - u(0, 1) \right ) + u(0,1).    
         \end{multline*}
         Since $u(0,0)-u(1,0)> 0$ and $u(1,1)- u(0,1) > 0$, by rearranging the terms we get 
         \begin{align*}
             P(Y=0 \given H=h, B=b) \geq \frac{u(1,1) - u(0,1)}{u(1,1) - u(1,0) + u(0,0) - u(0,1)}.
         \end{align*}
         Thus, 
         \begin{multline*}
             \pi^*(h,b)=0 \iff \mu(1 \given h, b) \leq \mu(0 \given h, b) \\ 
             \iff P(Y=0 \given H=h, B=b) \geq \frac{u(1,1) - u(0,1)}{u(1,1) - u(1,0) + u(0,0) - u(0,1)}
         \end{multline*}
\end{proof}

\subsection{Proof of Theorem~\ref{th:policy-reward}}

    For any $h\in \Hcal$, consider any $b\in[0,1], b\neq b^*(h)$, we distinguish two cases. If $b=b^*(h)+\epsilon$ for some $\epsilon>0$, then
    \begin{align}
        \mu(b\given h) &= \EE[\mu(1\given H,B) \cdot \mathbb{I}[ B> b]\given H  =h] \nonumber\\& + \EE[\mu(0\given H,B) \cdot \mathbb{I}[ b\geq B> b^*(H)]\given H=h]  + \EE[\mu(0\given H,B) \cdot \mathbb{I}[ B\leq b^*(H)]\given H=h] \nonumber
        \\
        &< \EE[\mu(1\given H,B) \cdot \mathbb{I}[ B> b]\given H=h] + \EE[\mu(1\given H,B) \cdot \mathbb{I}[ b\geq B> b^*(H)]\given H=h] \nonumber\\&  + \EE[\mu(0\given H,B) \cdot \mathbb{I}[ B\leq b^*(H)]\given H=h] 
        \label{eq:case_greater}\\
        &= \mu({b^*(h)}\given h) \nonumber
    \end{align}
    where Eq.~\ref{eq:case_greater} follows from Theorem~\ref{th:optimal-threshold} since $\mu(0\given h, b')< \mu(1\given h, b')$ for all $b'>b^*(h)$. On the other hand, if $b=b^*-\epsilon$ for some $\epsilon>0$, then
    \begin{align}
        \mu(b\given h) &= \EE[\mu(1\given H,B) \cdot \mathbb{I}[ B> b^*(H)]\given H=h] \nonumber\\& + \EE[\mu(1\given H,B) \cdot \mathbb{I}[ b^*(H)\geq B> b]\given H=h] + \EE[\mu(0\given H,B) \cdot \mathbb{I}[ B\leq b]\given H=h] \nonumber
        \\
        &\leq \EE[\mu(1\given H,B) \cdot \mathbb{I}[ B> b^*(H)]\given H=h] + \EE[\mu(0\given H,B) \cdot \mathbb{I}[ b^*(H)\geq B> b]\given H=h]\nonumber\\&  + \EE[\mu(0\given H,B) \cdot \mathbb{I}[ B\leq b]\given H=h]
        \label{eq:case_smaller}\\
        &= \mu({b^*(h)}\given h) \nonumber
    \end{align}
    where Eq.~\ref{eq:case_smaller} follows from Theorem~\ref{th:optimal-threshold} since $\mu(1\given h, b')\leq \mu(0\given h, b')$ for all $b'<b^*(h)$.
    In both cases, $\mu(b\given h)\leq \mu(b^*(h)\given h)$, thus $b^*(h)= \arg\max_{b\in \Bcal} \mu(b\given h)$.
%
\subsection{Proof of Theorem~\ref{th:regret-bound}}

We first state the well-known DKW inequality (or DKWM inequality) and its complimentary corollary, which we use to bound the error of the conditional probability estimators in Lemma~\ref{lem:estimators} used to estimate the conditional expected utility $\mu(b \given h)$ in Algorithm~\ref{alg:bandit-alg}.  
Then, we show that under perfect alignment, our
algorithm’s expected regret on the $|\Hcal|$ independent instances of the multiclass online learning problem is
equivalent to its expected regret on the binary online contextual learning problem with full feedback


\begin{theorem}(DKW inequality~\cite{dvoretzky1956asymptotic,massart1990tight}) \label{thm:dkwi}
Given a natural number $n$, let $Z_1, Z_2,\dots, Z_n$ be
real-valued independent and identically distributed random variables with cumulative distribution function $F(\cdot)$. Let $F_n$ denote the associated empirical distribution function defined by
\[
F_n(z) = \frac{1}{n} \sum_{i=1}^{n} \mathbb{I}[Z_i \leq z] \quad \text{for} \quad z\in \mathbb{R} 
\]
Then,  
\[ P\left(\sup_{z\in\Zcal}\ \ |F_n(z)-F(z)| > \epsilon \right) \leq 2 \exp^{-2n\epsilon^2}\]
\end{theorem}

\begin{corollary}\label{cor:dkwi-strict}
Given a natural number $n$, let $Z_1, Z_2,\dots, Z_n$ be
real-valued independent and identically distributed random variables with cumulative distribution function $F(\cdot)$. Let 
$
F^+(\cdot) = P(Z>z)
$ and $F^+_n$ denote the associated empirical distribution function defined by
\[
F^+_n(z) = \frac{1}{n} \sum_{i=1}^{n} \mathbb{I}[Z_i > z] \quad \text{for} \quad z\in \mathbb{R} 
\]
Then,  
\[ P\left(\sup_{z\in\Zcal}\ \ |F^+_n(z)-F^+(z)| > \epsilon \right) \leq 2 \exp^{-2n\epsilon^2}\]
\end{corollary}

\begin{proof}
    First, note that,
    \begin{multline}
        |F^+_n(z)-F^+(z)| = |F^+_n(z)+1-1-F^+(z)|
        = |-(1-F^+_n(z))+1-F^+(z)|
        \\= |-F_n(z)+F(z)|
        = |F_n(z)-F(z)|
    \end{multline}
    Substituting the above in Theorem~\ref{thm:dkwi} (DKW inequality), we get that
    \[ P\left(\sup_{z\in\Zcal}\ \ |F^+_n(z)-F^+(z)| > \epsilon \right) =
    P\left(\sup_{z\in\Zcal}\ \ |F_n(z)-F(z)|> \epsilon \right)
    \leq 2 \exp^{-2n\epsilon^2}\]
\end{proof}

\begin{lemma}\label{lem:estimators}
    For any fixed $h$, let $(B_i,Y_i)_{i\in[n]}$ be independent and identically distributed random variables with $B_i,Y_i \sim P(B,Y \given H=h)$. Consider the following estimators for any fixed $b\in \Bcal$
    \begin{align*}
        P_{n}(B\leq b \given H=h) &= \frac{1}{ n } \sum_{i\in[n]} \II \left[ B_i\leq b \right]
        \\
        P_{n}(Y=0, B\leq b \given H=h) &= \frac{1}{ n}\sum_{i\in[n]} \II \left[ B_i\leq b \right] (1-Y_i)
        \\
        P_{n}(Y=0, B> b \given H=h) &= \frac{1}{ n} \sum_{i\in[n]} \II \left[ B_i> b \right] (1-Y_i).
    \end{align*}
    Then, for any $\alpha\in (0,1)$, it holds that
    \begin{align}
        &P\left( \sup_{b\in \Bcal}\ |P(B\leq b \given H=h) - P_{n}(B\leq b \given H=h)| \right. \leq \left. \sqrt{\log(2/\alpha) / (2n) }\right) \geq 1- \alpha,
        \label{eq:bound-b}
        \\
        &P\left( \sup_{b\in \Bcal}\ |P(Y=0, B\leq b \given H=h) - P_{n}(Y=0, B\leq b \given H=h)| \right. \leq \left. \sqrt{\log(2/\alpha) / (2n) } \right)\nonumber \\ & \hspace{30em} \geq 1- \alpha,
        \label{eq:bound-leq}
        \\
        \text{and} 
        \quad &P\left( \sup_{b\in \Bcal}\ |P(Y=0, B> b \given H=h) - P_{n}(Y=0, B> b \given H=h)| \right. \leq \left. \sqrt{\log(2/\alpha) / (2n) } \right)\nonumber  \\ & \hspace{30em} \geq 1- \alpha
        \label{eq:bound-greater}
    \end{align}
\end{lemma}

\begin{proof}
    Consider any fixed $h\in \Hcal$.
    Eq.~\ref{eq:bound-b} follows directly from Theorem~\ref{thm:dkwi} by setting $Z_i=B_i, \forall i \leq [n]$, and $\epsilon = \sqrt{\log(2/\alpha) / (2n)}$.
    
    To show Eq.~\ref{eq:bound-leq}, consider
    the following random variables
    \[ 
    Z_i = B_i+Y_i \qquad \forall i\leq [n].
    \]
    Let $F(\cdot)$ denote the cumulative distribution function $P(Z\leq z) = P(B+Y \leq z \given H=h)$ and 
    \[F_n(z) = \frac{1}{n} \sum_{i: h_i=h} \II[B_i+Y_i \leq z] \quad \forall z \in \RR. \]
    Applying Theorem~\ref{thm:dkwi}, we have that, for any $\alpha \in (0,1)$, with probability $1-\alpha$ 
    \begin{equation}
        \sup_{z\in\RR} |F(z)-F_n(z)| \leq \sqrt{\log(2/\alpha) / (2n_h)} \label{eq:bound-F}
    \end{equation} 
    Further, note that for any $z\in [0,1]$ the following are equivalent statements
    \[ B+Y\leq z \iff Y=0 \text{ and } B\leq z \iff \II[B\leq z](1-Y)=1
    \]
    because if $Y=0 \text{ and } B\leq z$ then it implies that all statements are true (direction $\Longleftarrow$ $\implies$) and if either $Y=1$ or $B>z$ then it implies that all statements are false (direction $\implies$ $\Longleftarrow$).
    Hence, we have that, for $z\in[0,1]$, it holds that
    \[
    F(z) = P(B + Y \leq z \given H=h) =  P(Y=0, B\leq z \given H=h)
    \]
    and similarly
    \[F_n(z) = \frac{1}{n} \sum_{i\in[n]} \II[B_i+Y_i \leq z] = 
    \frac{1}{ n}\sum_{i\in[n]} \II \left[ B_i\leq z \right] (1-Y_i) = P_n(Y=0, B\leq z \given H=h).\]
    Thus, because of Eq.~\ref{eq:bound-F}, Eq.~\ref{eq:bound-leq} must hold for all $b\in[0,1]$.

    Analogously, to show Eq.~\ref{eq:bound-greater}, let
    \[
    Z_i = B_i(1-Y_i) \qquad \forall i\leq [n]
    \]
    let $F^+(\cdot)= P(Z> z) = P(B(1-Y) > z \given H=h)$ and 
    \[F^+_n(z) = \frac{1}{n} \sum_{i\in[n]} \II[B_i(1-Y_i) > z] \quad \forall z \in \RR. \]
    Applying Corollary~\ref{cor:dkwi-strict}, we have that for any $\alpha \in (0,1)$, with probability $1-\alpha$ 
    \begin{equation}
        \sup_{z\in\RR} |F^+(z)-F^+_n(z)| \leq \sqrt{\log(2/\alpha) / (2n_h)} \label{eq:bound-F+}
    \end{equation}
    Similarly, note that, for any $z\in [0,1]$ the following are equivalent statements
    \[ B(1-Y)> z \iff Y=0 \text{ and } B> z \iff \II[B>z](1-Y)=1.
    \]
    Hence, we have that
    \[
    F^+(z) = P(B(1-Y) > z \given H=h) = P(Y=0, B> z\given H=h)
    \]
    and 
    \[F^+_n(z) = \frac{1}{n} \sum_{i\in[n]} \II[B_i(1-Y_i) > z]
    = \frac{1}{ n} \sum_{i\in[n]} \II \left[ B_i> z \right] (1-Y_i) =P_n(Y=0, B> z \given H=h).\]
    Thus, because of Eq.~\ref{eq:bound-F+}, Eq.~\ref{eq:bound-greater} must hold for all $b\in[0,1]$.
\end{proof}

\begin{lemma}\label{lem:regret-per-h}
    Under perfect alignment, consider a selection algorithm that for confidence values $H_t, B_t$ picks threshold $\bar{B}_t(H_t) \in \Bcal$ and decides $A_t\in\{0,1\}$ according to this threshold, \ie, $A_t=\II[B_t > \bar{B}_t(H_t)]$. Let $P(\bar{B}_t(H_t) \given H_t)$ be the distribution over thresholds and $P(A_t \given H_t,B_t)$ be the distribution over decisions induced by this selection algorithm in time step $t$ and distributions $P(Y\given H,B)$ and $P(B\given H)$ (these distributions are relevant as they influence the history $(H_i,B_i,Y_i)_{i< t}$ that the algorithm observes). Then, we have that
    \begin{multline}
        \EE[R(T)]= \EE_{h_t,b_t\sim P(H,B), a_t\sim P(A_t \given H_t,B_t)}\left[\sum_{t=1}^T \mu(\pi^*(h_t,b_t) \given h_t,b_t)- \mu(a_t \given h_t,b_t)\right]
        \\=\EE_{h_t\sim P(H), \estb_t(h_t) \sim P( \bar{B}_t(H_t) \given H_t=h_t)}\left[\sum_{t=1}^T \mu(b^*(h_t) \given h_t)-\mu(\estb_t(h_t) \given h_t)\right].
    \end{multline} 
\end{lemma}

\begin{proof}
By the law of total expectation
    \begin{align*}
        \EE&[R(T)] 
        = \EE_{h_t,b_t\sim P(H,B),a_t\sim P(A_t \given H_t,B_t)}\left[\sum_{t=1}^T \mu(\pi^*(h_t,b_t) \given h_t,b_t)- \mu(a_t \given h_t,b_t)\right]
        \\&= \EE_{h_t\sim P(H)}\left[\sum_{t=1}^T \EE_{b_t\sim P(B \given H=h_t), a_t\sim P(A \given H=h_t,B=b_t)} \left[\mu(\pi^*(h_t,b_t) \given h_t,b_t)- \mu(a_t \given h_t,b_t)\right] \right]
    \end{align*}
    For a fixed $H=h_t$ and any fixed threshold $\estb$, the choice of decision $A_t$ is defined deterministically by the fixed threshold $\estb$. Thus, we have that 
    \begin{align*}
        \EE_{b_t\sim P(B \given H=h_t)} &\left[\mu(A_t \given H_t=h_t,B_t=b_t) \given \bar{B}_t=\estb  \right] \\&= 
        \EE_{b_t\sim P(B \given H=h_t)} \left[\EE[\mu(1\given h_t,b_t) \cdot \mathbb{I}[ b_t> \estb] ] +\EE[\mu(0\given h_t,b_t) \cdot \mathbb{I}[ b_t\leq  \estb]]\right] 
        \\& = \mu(\estb\given h_t) 
    \end{align*}
    Since, at each timestep $t$, $\estb_t(h_t)$ is chosen by the selection algorithm given $h_t$ under distribution $P(\bar{B}_t(H_t) \given H_t=h_t)$, we can substitute the induced distribution $P(A_t \given H_t=h_t,B_t=b_t)$ for decision $A_t$ by this distribution for threshold $\bar{B}_t(H_t)$ and obtain
    \begin{align*}
        \EE&_{b_t\sim P(B \given H=h_t), a_t\sim P(A_t \given H_t=h_t,B_t=b_t)} \left[\mu(a_t \given h_t,b_t) \right] \\
        &= \EE_{ \estb_t\sim P(\bar{B}_t(H_t) \given H_t=h_t)} [\EE_{b_t\sim P(B \given H=h_t)}\left[\mu(A_t \given h_t,b_t) \given \bar{B}_t(H_t)=\estb_t(h_t) \right]
        \\&= \EE_{ \estb_t\sim P(\bar{B}_t(H_t)  \given H_t=h_t)}[\mu(\estb_t(h_t)  \given h_t)]
    \end{align*}
    
    Since the optimal policy $\pi^*$ correspond to a threshold function with fixed optimal thresholds $b^*(h)$ for each $h\in \Hcal$, we get
    \begin{multline*}
        \EE_{b_t\sim P(B \given H=h_t)} \left[\mu(\pi(h_t,b_t) \given h_t,b_t) \right] \\=
    \EE_{b_t\sim P(B \given H=h_t)} \left[\mu(A^* \given h_t,b_t) \given B^*(H_t)=b^*(h_t) \right] 
    = \mu(b^*(h_t) \given h_t)
    \end{multline*}
    where optimal decision $A^*$ given confidence $H_t$ depends on optimal threshold $B^*(H_t)$ with $A^*= \II[b_t>B^*(H_t)]$.
    Thus, we have that
    \begin{align*}
        \EE[R(T)] 
        &= \EE_{h_t\sim P(H)}\left[\sum_{t=1}^T \mu(b^*(h_t) \given h_t)-\EE_{ \estb_t\sim P(\bar{B}_t \given H_t)}[\mu(\estb_t(h_t) \given h_t)]\right]
        \\&= \EE_{h_t\sim P(H), \estb_t\sim P(\bar{B}_t(H_t) \given H_t=h_t)}\left[\sum_{t=1}^T \mu(b^*(h_t) \given h_t)-\mu(\estb_t(h_t) \given h_t)\right]
    \end{align*}

\end{proof}

With the above auxiliary Lemmas, we are ready to prove Theorem~\ref{th:regret-bound}. 


By Lemma~\ref{lem:regret-per-h} we have that
\begin{equation*}
    \EE[R(T)] = \EE_{h_{t}\sim P(H), \estb_t\sim P(\bar{B}_t(H_t) \given H_t=h_t)} 
    \left [ \sum_{t=1}^T \mu(b^{*}(h_t) \given h_{t}) - \mu(\bar{b}_{t}(h_t) \given h_{t}) \right ],
\end{equation*}
where $\estb_t(h_t)$ is the value of the threshold selected by Algorithm~\ref{alg:bandit-alg} at time step $t$ and $P(\bar{B}_t(H_t) \given H_t=h_t)$ is the corresponding distribution.
%
In what follows, we will show that selecting $\estb_t(h_t)$ using Algorithm~\ref{alg:bandit-alg} achieves 
expected regret $\EE(R(T))= \Ocal \left (\sqrt{|\Hcal|T \log T} \right)$.

Let us first define an estimator for the conditional expected utility $\mu(b \given h)$ achieved by the decision policy $\pi$ defined by a
threshold function with threshold $b$ given the human confidence $H= h$, where as a reminder
\begin{align*}
    \mu(b \given h) &= \EE[\mu(1 \given h, B) \cdot \II[B > b ] \given H=h] + \EE[\mu(0 \given h, B)\cdot \II[B\leq b] \given H=h].
\end{align*}
By Eq.~\ref{eq:mu-1} and Eq.~\ref{eq:mu-0} the above becomes
\begin{align*}
    \mu(b \given h) &= \EE \left[ \left( P(Y=0 \given H=h, B)(u(1,0) - u(1,1)) + u(1,1) \right)\cdot\II[B > b] \given H=h\right]\\
    &+ \EE \left[ \left( P(Y=0 \given H=h, B)(u(0,0) - u(0,1)) + u(0,1) \right)\cdot\II[B \leq b] \given H=h \right]\\
    &= P(Y=0, B > b \given H=h) (u(1,0) - u(1,1))\\
    &+ P(Y=0, B \leq b \given H=h) (u(0,0) - u(0,1))\\
    &+ P(B \leq b\given H=h)(u(0,1) - u(1,1)) + u(1,1).
\end{align*}
We can now define the estimator $\estmu_t(b \given h)$ at time step $t$ as
\begin{align}\label{eq:estimator}
    \estmu_{t}(b \given h) &= P_{n_t}(Y=0, B > b \given H=h) \cdot(u(1,0)-u(1,1)) \nonumber\\
    &+ P_{n_t}(Y=0, B \leq b \given H=h) \cdot(u(0,0)-u(0,1)) \nonumber\\ 
    &+ P_{n_t}(B \leq b \given H=h) \cdot (u(0,1)-u(1,1)) + u(1,1), 
\end{align}
where, as in Lemma~\ref{lem:estimators},
\begin{align*}
    P_{n_t}(Y=0, B > b\given H=h ) &= \frac{1}{n_{t}(h)}\sum_{i\leq t: h_{i} = h} \II \left[ b_{i} > b \right](1 - y_{i}),\\
    P_{n_t}(Y=0 , B \leq b \given H=h) &= \frac{1}{n_{t}(h)}\sum_{i\leq t: h_{i} = h } \II \left[b_{i} \leq b\right](1 - y_{i}),\\ 
    P_{n_t}(B \leq b \given H = h) &= \frac{1}{n_{t}(h)}\sum_{i\leq t: h_{i} = h} \II \left[ b_i \leq b \right],
\end{align*}
with $n_{t}(h) = \sum_{i=1}^{t}\II \left [ h_i = h \right]$. 

In what follows, we will first express the regret bound in terms of the regret under the clean and the bad event for the above estimator.
We define the clean event as
\begin{equation*}
\Ecal = \left \{ \left | \estmu_{t}(b \given h)    - \mu(b \given h) \right| \leq \texttt{rad}_{t}(h), \; \forall b \in \Bcal, \forall h \in \Hcal, 1\leq t \leq T  \right \},
\end{equation*}
where $\texttt{rad}_{t}(h)= 3\sqrt{\frac{\log(6|\Hcal|T^3)}{2n_{t}(h)}}$.
%
By Lemma~\ref{lem:regret-per-h} it holds 
\begin{align}\label{eq:total-exp-regret-ht}
    \EE(R(T))
    &= P(\Ecal) \cdot \EE_{h_t,\estb_t(h_t)\sim P(H,\bar{B}_t(H_t))} \left [\sum_{t=1}^T \mu(b^{*}(h_t) \given h_t) - \mu(\estb_t(h_t) \given h_t) \given \Ecal \right ] \nonumber\\
    &+ P(\Ecal^{c}) \cdot \EE_{h_t,\estb_t(h_t)\sim P(H,\bar{B}_t(H_t))} \left [\sum_{t=1}^T \mu(b^{*}(h_t) \given h_t) - \mu(\estb_t(h_t) \given h_t) \given \Ecal^{c}\right ], \nonumber\\
    &\leq \EE_{h_t,\estb_t(h_t)\sim P(H,\bar{B}_t(H_t))} \left [\sum_{t=1}^T \mu(b^{*}(h_t) \given h_t) - \mu(\estb_t(h_t) \given h_t) \given \Ecal \right ] + P(\Ecal^{c})\cdot T 
\end{align}
where $\Ecal^{c}$ is the complementary of the clean event $\Ecal$ (and we have written the joint distribution $ P(H,\bar{B}_t(H_t))$ for brevity).
By Lemma~\ref{lem:estimators} for $\alpha = \frac{1}{3|\Hcal|T^{3}}$, the union bound over the three estimators, all $h \in \Hcal$ and all $t \in \{1,.., T\}$, and given  $u(\cdot) \in [0,1]$, we have that $P(\Ecal) \geq 1 - \frac{3|\Hcal|T}{3|\Hcal|T^3} = 1 - \frac{1}{T^{2}}$ and consequently $P(\Ecal^{c}) \leq \frac{1}{T^2}$. We can now rewrite Eq.~\ref{eq:total-exp-regret-ht} as
\begin{align}\label{eq:regret-clean-plus-bad}
      \EE(R(T))\leq \EE_{h_t,\estb_t(h_t)\sim P(H,\bar{B}_t(H_t))} \left [\sum_{t=1}^T \mu(b^{*}(h_t) \given h_t) - \mu(\estb_t(h_t) \given h_t) \given \Ecal \right ] + \frac{1}{T}.
\end{align}

We proceed by deriving an upper bound for \[\EE_{h_t,\estb_t(h_t)\sim P(H,\bar{B}_t(H_t))} \left [\sum_{t=1}^T \mu(b^{*}(h_t) \given h_t) - \mu(\estb_t(h_t) \given h_t) \given \Ecal \right ].\]
%
Let $\estb_t(h_t) = \argmax_{b \in \Bcal} \bar{\mu}_{t}(b \given h_t)$, \ie, the arm pulled by Algorithm~\ref{alg:bandit-alg} in time step $t$. For $\estb_t(h_t)$ at time step $t$, assuming that the clean event holds, we have: 
\begin{align}~\label{eq:ucb-max-subopt}
    \estmu_{t}(\estb_{t}(h_t) \given h_{t}) \geq \estmu_{t}(b^{*}(h_t) \given h_{t}).
\end{align}
Under the clean event it holds that $\mu(b \given h_{t}) + \texttt{rad}_t(h_t)\geq \estmu_{t}(b \given h_{t})$ and that $\estmu_{t}(b \given h_{t} ) \geq \mu(b \given h_t) - \texttt{rad}_t(h_t)$ for any arm $b$. Therefore, given Eq.~\ref{eq:ucb-max-subopt} we have
    \begin{align*}
        \mu(\estb_{t}(h_t) \given h_{t}) +  \texttt{rad}_t(h_t) \geq \estmu_{t}(\estb_{t}(h_t) \given h_t)  \geq \estmu_{t}(b^{*}(h_t) \given h_t) \geq \mu(b^{*}(h_t) \given h_t) -\texttt{rad}_t(h_t) ,
    \end{align*}
    and thus 
    \[\mu(b^*(h_t) \given h_t) - \mu(\estb_t(h_t) \given h_t) \leq 2 \cdot \texttt{rad}(h_t),\]
    which is the regret due to pulling a suboptimal arm $\estb_t(h_t)$ at a single time step $t$ given confidence $h_{t}$. For the cumulative regret given confidence $h$ up to time step $t$ given the above, we have 
    \begin{align}\label{eq:regret-fixed-h}
        \sum_{i=1}^{t} \II[h_i=h] \cdot \left ( \mu(b^*(h_t) \given h)-\mu(\estb_i(h_t) \given h) \right )  &\leq  \sum_{i =1}^{t} \II[h_i=h ]\cdot 6\sqrt{\frac{\log(6|\Hcal|T^{3})}{2 n_{i}(h)}} 
        \\&= \sqrt{18\log(6|\Hcal|T^{3})}\sum_{i = 1}^{n_{t}(h)} \frac{1}{\sqrt{i}} \nonumber
        \\&= \sqrt{18\log(6|\Hcal|T^{3})} \sum_{i =1}^{n_{t}(h)} \int_{i-1}^i \frac{dx}{\sqrt{i}}
        \\ &\leq \sqrt{18\log(6|\Hcal|T^{3})} \sum_{i =1}^{n_{t}(h)} \int_{i-1}^i\frac{dx}{\sqrt{x}} \nonumber
        \\&= \sqrt{18\log(6|\Hcal|T^{3})}\int_{0}^{n_{t}(h)}\frac{dx}{\sqrt{x}} \nonumber
        \\& = 2\sqrt{18n_{t}(h)\log(6|\Hcal|T^{3})}. 
    \end{align}
Note that
\begin{align}\label{eq:regret-exp-over-h}
     \EE&_{h_t,\estb_t(h_t)\sim P(H,\bar{B}_t(H_t))}\left[\sum_{t=1}^T \mu(b^*(h_t) \given h_t)-\mu(\estb_t(h_t) \given h_t) \given \Ecal \right] \nonumber
     \\&= \EE_{h_t,\estb_t(h_t)\sim P(H,\bar{B}_t(H_t))}\left[\sum_{t=1}^T \sum_{h \in \Hcal} \II [h_{t} = h] \cdot (\mu(b^*(h_t) \given h)-\mu(\estb_t(h_t) \given h)) \given \Ecal \right] \nonumber\\
     &= \EE_{h_t,\estb_t(h_t)\sim P(H,\bar{B}_t(H_t))}\left[ \sum_{h \in \Hcal} \sum_{t=1}^T \II [h_{t} = h] \cdot (\mu(b^*(h_t) \given h)-\mu(\estb_t(h_t) \given h)) \given \Ecal \right]  
\end{align}
By Eq.~\ref{eq:regret-fixed-h}, it follows that $\sum_{t=1}^{T} \II[h_t=h] \cdot \left ( \mu(b^*(h)  \given h)-\mu(\estb_t(h)  \given h) \right )\leq 2\sqrt{18n_{T}(h)\log(6|\Hcal|T^{3})}$. As a result, we can bound Eq.~\ref{eq:regret-exp-over-h} as follows
\begin{align}\label{eq:regret-exp-h-upper}
    \EE_{h_t,\estb_t(h_t) \sim P(H,\bar{B}_t(h_t))}&\left[\sum_{t=1}^T \mu(b^*(h_t) \given h_t)-\mu(\estb_t(h_t)  \given h_t) \given \Ecal \right] \\&\leq  \EE_{h_t,\estb_t(h_t) \sim P(H,\bar{B}_t(H_t))}\left[\sum_{h \in \Hcal}   2\sqrt{18n_{T}(h)\log(6|\Hcal|T^{3})} \given \Ecal \right] \nonumber\\
    &= \EE_{h_t,\estb_t(h_t)\sim P(H,\bar{B}_t(H_t))}\left[2\sqrt{18\log(6|\Hcal|T^{3})}\sum_{h \in \Hcal}   \sqrt{n_{T}(h)} \given \Ecal \right].
\end{align}

By Jensen'{}s inequality it holds
\begin{align*}
    \frac{\sum_{h \in \Hcal} \sqrt{n_{T}(h)}}{|\Hcal|}  \leq \sqrt{\frac{\sum_{h \in \Hcal}n_{T}(h)}{|\Hcal|}}
\end{align*}
and since $\sum_{h \in \Hcal}n_{T}(h) = T$, we have
\begin{align*}
    \sum_{h \in \Hcal} \sqrt{n_{T}(h)}\leq |\Hcal|\sqrt{\frac{T}{|\Hcal|}} = \sqrt{|\Hcal|T }.
\end{align*}
We can now bound Eq.~\ref{eq:regret-exp-h-upper} using the above
\begin{multline*}
    \EE_{h_t,\estb_t(h_t)\sim P(H,\bar{B}_t(H_t))}\left[\sum_{t=1}^T \mu(b^*(h_t) \given h_t)-\mu(\estb_t(h_t) \given h_t) \given \Ecal \right] 
    \\\leq \EE_{h_t,\estb_t(h_t)\sim P(H,\bar{B}_t(H_t))}\left[2\sqrt{18|\Hcal|T\log(6|\Hcal|T^{3})} \given \Ecal \right] 
    = 2\sqrt{18|\Hcal|T\log(6|\Hcal|T^{3})}.
\end{multline*}
As long as $|\Hcal| < T$, it holds that under the clean event $\Ecal$
\begin{align*}
     \EE_{h_t,\estb_t\sim P(H,\bar{B})}\left[\sum_{t=1}^T \mu(b^* \given h_t)-\mu(\estb_t \given h_t) \given \Ecal \right] = \Ocal \left(\sqrt{ |\Hcal|T \log T }\right)
\end{align*}
Finally, using the above and Eq.~\ref{eq:regret-clean-plus-bad} we have  
\begin{align*}
     \EE_{h_t,\estb_t(h_t)\sim P(H,\bar{B}_t(H_t))}\left[\sum_{t=1}^T \mu(b^*(h_t) \given h_t)-\mu(\estb_t(h_t) \given h_t)\right] = \Ocal \left(\sqrt{ |\Hcal|T \log T }\right)
\end{align*}

\subsection{Proof of Theorem~\ref{th:policy-reward-2d}}

For any function $\ell\in \Lcal$, we have that by the law of total expectation
    \begin{align*}
        \mu(\ell) 
        &= \EE[\mu(1\given H,B) \cdot \mathbb{I}[ B> \ell(H)]] + \EE[\mu(0\given H,B) \cdot \mathbb{I}[ B\leq \ell(H)]]
        \\&= \EE_{H} [\EE[\mu(1\given H,B) \cdot \mathbb{I}[ B> \ell(H)]\given H] + \EE[\mu(0\given H,B) \cdot \mathbb{I}[ B\leq \ell(H)]\given H]]
        \\&= \EE_{H} [ \mu(\ell(H) \given H)]
    \end{align*}
    From Theorem~\ref{th:policy-reward}, we know that for any $h\in \Hcal$, $b^*(h) = \arg \max_{b \in \Bcal } \mu(b \given h)$. Thus, $\ell(h)=b^*(h)$ for all $h\in \Hcal$ maximizes the value of $\EE_{H} [ \mu(\ell(H) \given H)]$. 
    
\subsection{Proof of Theorem~\ref{th:regret-bound-2d} }

We proceed the proof analogously to the proof of Theorem~\ref{th:regret-bound}.
Here, we use a non-trivial extension
of the DKW inequality DKW-like inequality and its complimentary corollary, stated and proven in Section~\ref{app:dkw-like-inequality}, to bound the error of the conditional probability estimators in Lemma~\ref{lem:estimators-2d}---these estimators estimate the expected utility $\mu(\ell)$ as defined in Eq.~\ref{eq:policy-reward-2d}. We use this to show that in each time step $t$ with high probability the utility of the selected function (arm) $\estell_t$ is within certain bounds of the utility of the optimal function $\ell^*$. 
Then, we show that, under perfect alignment, the expected regret of the binary online contextual learning problem with full feedback is equivalent to the expected regret of the multiclass online contextual learning problem over functions (arms) $\ell \in \Lcal:= \{\ell(\cdot) |\ell: \Hcal \to\Bcal\}$  where decisions $A=\{0,1\}$ are defined by a threshold function $A=\II[B_t>\ell(H_t)]$ with $H_t, B_t\sim P(H,B)$. 



More formally, we define the class of policies induced by functions in $\Lcal$ as follows

\begin{definition}\label{def:Dcal-policy}
    Given a fixed $\vecb$ in vector form, 
    $\vecb = (\vecb_h)_{h\in \Hcal} \in \Bcal^\Hcal$, 
    let $\Delta_{\vecb}$ denote the following policy for 
    $h, b\in \Hcal \times \Bcal$
    \[\Delta_{\vecb}(h,b) = \II[ b\leq \vecb_h].\]
    Let $\Lcal:= \{\ell(\cdot) |\ell: \Hcal \to\Bcal\}$. We define the class $\Dcal_{\Lcal}$ of policies  induced by $\Lcal$ as
\[
 \Dcal_{\Lcal} := \{ \Delta_{\vecb} \given \ell \in \Lcal: \vecb_h = \ell(h) \text{ for all } h\in \Hcal \} = \{ \Delta_{\vecb} \given \vecb \in \Bcal^\Hcal\}
\]
\end{definition}

\begin{lemma}\label{lem:estimators-2d}
    Let $(H_i,B_i,Y_i)_{i\in[n]}$ be independent and identically distributed random variables with $H_i,B_i,Y_i \sim P(H,B,Y)$. Let $\Bcal$ be countable (and recall $\Hcal$ is discrete) and let  $\Dcal_{\Lcal}$ be the class of (countable) policies as in Definition~\ref{def:Dcal-policy}. Consider the following estimators for any fixed $\Delta_{\vecb} \in \Dcal$
    \begin{align*}
        P_{n}(\Delta_{\vecb}(H,B)=1) &= \frac{1}{ n } \sum_{i\in[n]} \II \left[ B_i\leq \vecb_{H_i} \right]
        \\
        P_{n}(Y=0, \Delta_{\vecb}(H,B)=1 ) &= \frac{1}{ n}\sum_{i\in[n]} \II \left[ B_i\leq \vecb_{H_i} \right] (1-Y_i)
        \\
        P_{n}(Y=0, \Delta_{\vecb}(H,B)=0 ) &= \frac{1}{ n} \sum_{i\in[n]} \II \left[ B_i> \vecb_{H_i} \right] (1-Y_i).
    \end{align*}
    Then, there exist constant $c_3>0$ such that, for any $\alpha\in (0,1)$ with $\log(1/\alpha)> c_3 \sqrt{|\Hcal|}$, it holds that
    \begin{align}
        &P\left( \sup_{\Delta_{\vecb} \in \Dcal_{\Lcal}}\ |P(\Delta_{\vecb}(H,B)=1) - P_{n}(\Delta_{\vecb}(H,B)=1)|  \leq \sqrt{4c_2\sqrt{|\Hcal|}\log(1/\alpha) / n}\right) \geq 1- \alpha,
        \label{eq:bound-b-2d}
        \\
        &P\left( \sup_{\Delta_{\vecb}\in \Dcal_{\Lcal}}\ |P(Y=0, \Delta_{\vecb}(H,B)=1) - P_{n}(Y=0, \Delta_{\vecb}(H,B)=1)|  \leq  \sqrt{4c_2\sqrt{|\Hcal|}\log(1/\alpha) / n} \right) \nonumber \\ & \hspace{30em} \geq 1- \alpha, \text{ and }
        \label{eq:bound-leq-2d}
        \\
        &P\left( \sup_{\Delta_{\vecb}\in \Dcal_{\Lcal}}\ |P(Y=0, \Delta_{\vecb}(H,B)=0) - P_{n}(Y=0, \Delta_{\vecb}(H,B))=0|  \leq \sqrt{4c_2\sqrt{|\Hcal|}\log(1/\alpha) / n} \right) \nonumber \\ & \hspace{32em} \geq 1- \alpha
        \label{eq:bound-greater-2d}
    \end{align}
\end{lemma}
\begin{proof}
    The proof follows the proof of Theorem~\ref{lem:estimators} closely using Theorem~\ref{thm:dkwi-thresholds} and Corollary~\ref{cor:dkwi-thresholds+} instead of the DKW inequality.
    Accordingly, to use the previous results to prove Eq.~\ref{eq:bound-b-2d} to~\ref{eq:bound-greater-2d}, we use that the probability of an indicator random variable obtaining value of $1$ is equal the expectation of the random variable, that is, for example 
    \[
        P(\Delta_{\vecb}(H,B)=1) = \EE_{H,B\sim P(H,B)}[\Delta_{\vecb}(H,B)]
    \]
    
    Eq.~\ref{eq:bound-b-2d} follows directly from Theorem~\ref{thm:dkwi-thresholds} by setting $K_i=H_i$ and $X_i=B_i, \forall i \leq [n]$, and $\epsilon = \sqrt{4c_2\sqrt{|\Hcal|}\log(1/\alpha) / n}$. More formally, because $\Delta_{\vecb}(H_i,B_i) = \II \left[ B_i\leq b_{H_i} \right]$, Eq.~\ref{eq:bound-b-2d} is equivalent to
    \begin{multline*}
        P\left( \sup_{\Delta_{\vecb} \in \Dcal_{\Lcal}}\ \left| \EE_{H,B\sim P(H,B)}[\Delta_{\vecb}(H,B)] - \frac{1}{ n } \sum_{i\in[n]} \Delta_{\vecb}(H_i,B_i) \right|  \leq  \sqrt{4c_2\sqrt{|\Hcal|}\log(1/\alpha) / n} \right) \\\geq 1- \alpha.
    \end{multline*}
    which holds by Theorem~\ref{thm:dkwi-thresholds} when $\epsilon = \sqrt{4c_2\sqrt{|\Hcal|}\log(1/\alpha) / n} > c_1 \sqrt{|\Hcal|/n}$, that is, 
    for $\log(1/\alpha)> c_3 \sqrt{|\Hcal|} $ where $c_3:=\frac{c_1^2}{4c_2}$.
    
    To show Eq.~\ref{eq:bound-leq-2d}, consider
    the following random variables $K_i=H_i$ and 
    \[ 
    X_i = B_i+Y_i \qquad \forall i\leq [n].
    \]
    where $\Xcal = \Bcal + \Ycal$ is countable as $\Bcal$ is countable and $\Ycal=\{0,1\}$.
    Consider any policy $\Delta_{\vec{x}}$ in countable class $\Dcal$ as in Definition~\ref{def:def-2Dthresholds} where $\vec{x} \in \Xcal^\Hcal$ and the following estimator
    \[ \frac{1}{n} \sum_{i\in [n]} \Delta_{\vec{x}}(H_i, B_i+Y_i) = \frac{1}{n} \sum_{i\in [n]} \II[B_i+Y_i \leq \vec{x}_{H_i}] \]
    Applying Theorem~\ref{thm:dkwi-thresholds}, we have that, for any $\alpha \in (0,1)$ with $\log(1/\alpha)> c_3 \sqrt{|\Hcal|}$, with probability $1-\alpha$ 
    \begin{equation}
        \sup_{\Delta_{\vec{x}} \in \Dcal} \left|\EE_{H,B,Y\sim P(H,B,Y)}[\Delta_{\vec{x}}(H,B+Y)] - \frac{1}{n} \sum_{i\in [n]} \Delta_{\vec{x}}(H_i, B_i+Y_i)\right| \leq \sqrt{4c_2\sqrt{|K|}\log(1/\alpha) / n} \label{eq:bound-F-2d}
    \end{equation} 
    Further, note that $\Bcal \subset (\Ycal + \Bcal) =\Xcal$ and, for any $\vecb \in \Bcal^\Hcal$,  the following are equivalent statements
    \[ B+Y\leq \vecb_H \iff Y=0 \text{ and } B\leq \vecb_H \iff \II[B\leq \vecb_H](1-Y)=1
    \]
    because if $Y=0 \text{ and } B\leq \vecb_H$ then it implies that all statements are true (direction $\Longleftarrow$ $\implies$) and if either $Y=1$ or $B>\vecb_H$ then it implies that all statements are false (direction $\implies$ $\Longleftarrow$).
    Hence, we have that, for any $\Delta_{\vecb} \in \Dcal_{\Lcal}$ and $\Delta_{\vec{x}} \in \Dcal$ with $\vec{x}=\vecb$, it holds that
    \[
     \EE_{H,B,Y\sim P(H,B,Y)}[\Delta_{\vec{x}}(H,B+Y)] = P(B+Y\leq \vecb_H) =  P(Y=0, \Delta_{\vecb}(H,B)=1)
    \]
    and similarly
    \begin{multline*}
        \frac{1}{n} \sum_{i\in [n]} \Delta_{\vec{x}}(H_i, B_i+Y_i) = \frac{1}{n} \sum_{i\in[n]} \II[B_i+Y_i \leq \vec{x}_{H_i}] \\= 
    \frac{1}{ n}\sum_{i\in[n]} \II \left[ B_i\leq \vec{x}_{H_i} \right] (1-Y_i) = P_n(Y=0, \Delta_{\vecb}(H,B)=1 ).
    \end{multline*}
    Thus, because of Eq.~\ref{eq:bound-F-2d} and $\Bcal \subset \Xcal$, Eq.~\ref{eq:bound-leq-2d} must hold for all $\Delta_{\vecb} \in \Dcal_{\Lcal}$.

    Analogously, to show Eq.~\ref{eq:bound-greater-2d}, let $K_i=H_i$ and 
    \[ 
    X_i = B_i(1-Y_i) \qquad \forall i\leq [n].
    \]
    where $\Xcal = \Bcal \cup \{0\}$ is countable.
    Consider any policy $\Delta_{\vec{x}}$ in countable class $\Dcal$ where $\vec{x} \in \Xcal^\Hcal$ and the following estimator
    \[ \frac{1}{n} \sum_{i\in [n]} \left( 1- \Delta_{\vec{x}}(H_i, B_i+Y_i) \right) = \frac{1}{n} \sum_{i\in [n]} \II[B_i(1-Y_i) > \vec{x}_{H_i}] \]
    Applying Corollary~\ref{cor:dkwi-thresholds+}, we have that, for any $\alpha \in (0,1)$  with $\log(1/\alpha)> c_3 \sqrt{|\Hcal|}$, with probability $1-\alpha$ 
    \begin{multline}
        \sup_{\Delta_{\vec{x}} \in \Dcal} \left|\EE_{H,B,Y\sim P(H,B,Y)}[1-\Delta_{\vec{x}}(H,B(1-Y))] - \frac{1}{n} \sum_{i\in [n]} (1-\Delta_{\vec{x}}(H_i, B_i(1-Y_i)))\right|\\ \leq \sqrt{4c_2\sqrt{|K|}\log(1/\alpha) / n} \label{eq:bound-F+-2d}
    \end{multline}
    Similarly, note that, for $\vecb \in \Bcal^\Hcal$,  the following are equivalent statements
    \[ B(1-Y)> \vecb_H \iff Y=0 \text{ and } B> \vecb_H \iff \II[B> \vecb_H](1-Y)=1
    \]
    Hence, we have that, for any $\Delta_{\vecb} \in \Dcal_{\Lcal}$ and $\Delta_{\vec{x}} \in \Dcal$ with $\vec{x}=\vecb$, it holds that
    \[
     \EE_{H,B,Y\sim P(H,B,Y)}[1-\Delta_{\vec{x}}(H,B(1-Y)] = P(B(1-Y)\leq \vecb_H) =  P(Y=0, \Delta_{\vecb}(H,B)=0)
    \]
    and 
    \begin{multline}
        \frac{1}{n} \sum_{i\in [n]} (1-\Delta_{\vec{x}}(H_i, B_i(1-Y_i))) = \frac{1}{n} \sum_{i\in[n]} \II[B_i(1-Y_i) > \vec{x}_{H_i}] \\
        = \frac{1}{ n}\sum_{i\in[n]} \II \left[ B_i> \vec{x}_{H_i} \right] (1-Y_i) = P_n(Y=0, \Delta_{\vecb}(H,B) =0).
    \end{multline}
    Thus, because of Eq.~\ref{eq:bound-F+-2d} and $\Bcal \subseteq \Xcal$, Eq.~\ref{eq:bound-greater-2d} must hold for all $\Delta_{\vecb} \in \Dcal_{\Lcal}$.
\end{proof}

\begin{lemma}\label{lem:error-bound-2d}
    Let $\Bcal$ be countable (and recall $\Hcal$ is discrete and $u: \Acal \times \Ycal \to [0,1]$). At time step $t$, let $\estell_t \in \Lcal$ be the function with highest average utility, \ie,
    $\estell_t = \arg \max_{\ell \in \Lcal} \estmu_t(\ell)$
    with 
    \begin{align}\label{eq:estimator-2d-app}
    \estmu_{t}(\ell) &= P_{t}(Y=0, B > \ell(H)) \cdot(u(1,0)-u(1,1)) \nonumber\\
    &+ P_{t}(Y=0, B \leq \ell(H)) \cdot(u(0,0)-u(0,1)) \nonumber\\ 
    &+ P_{t}(B \leq \ell(H)) \cdot (u(0,1)-u(1,1)) + u(1,1), 
\end{align}
where $P_t(\cdot)$ is the probability estimated with observations $(H_{t'},B_{t'},Y_{t'})_{t'< t}$ as defined in Lemma~\ref{lem:estimators-2d}.
Then, there exists constant $c_3>0$ such that, for any $\alpha\in(0,1)$ with $\log(3/\alpha)> c_3 \sqrt{|\Hcal|}$, it holds that with probability $1- \alpha$
    \[
    \mu(b^*)- \mu(\estell_t)\leq  6\cdot \sqrt{4c_2\sqrt{|\Hcal|}\log(3/\alpha) / t} .
    \]
\end{lemma}
\begin{proof}
    Let $\texttt{rad}_t := \sqrt{4c_2\sqrt{|\Hcal|}\log(3/\alpha) / t}$.
    Because each function in $\Lcal$ induces a policy in $\Dcal_{\Lcal}$,
    the supremum over $\Lcal$ can be substituted by the supremum over $\Dcal_{\Lcal}$. Using Lemma~\ref{lem:estimators-2d} for $\alpha\in(0,1)$ with $\log(3/\alpha)> c_3 \sqrt{|\Hcal|}$, we have that 
    \begin{align*}
        P&\left( \sup_{\ell \in \Lcal}\ |P(B>\ell(H)) - P_{t}(B>\ell(H))|  \leq \texttt{rad}_t \right)
        \\ &=
        P\left( \sup_{\Delta_{\vecb} \in \Dcal_{\Lcal}}\ |P(\Delta_{\vecb}(H,B)=1) - P_{n}(\Delta_{\vecb}(H,B)=1)|  \leq \sqrt{4c_2\sqrt{|\Hcal|}\log(3/\alpha) / t}\right) 
        \\ &\geq 1- \alpha/3.
    \end{align*}
    Analogously, we can show the same result for the other estimators in $\estmu_t(\ell)$.
    Now let $\estell_t = \argmax_{\ell \in \Lcal} \estmu_t(\ell)$ and let $\ell^* = \argmax_{\ell \in \Lcal} \mu(\ell)$, 
    then we have that
    \begin{align}~\label{eq:ucb-max-subopt-2d}
    \estmu_{t}(\estell_{t}) \geq \estmu_{t}(\ell^{*}).
    \end{align}
    Because $u(\cdot) \in [0,1]$, the error of each of the three estimators cannot be amplified when estimating $\estmu_{t}(\ell)$ but only aggregated. Hence,
    using a union bound over the three estimators in $\estmu_{t}(\ell)$, with probability greater that $1-\alpha$, it holds that $\mu(\ell) + 3\texttt{rad}_t \geq \estmu_{t}(\ell)$ and that $\estmu_{t}(\ell) + 3\texttt{rad}_t\geq \mu(\ell)$ for any arm $\ell \in \Lcal$. Therefore, given Eq.~\ref{eq:ucb-max-subopt-2d}, we have
    \begin{align*}
        \mu(\estell_t) + 6 \cdot \texttt{rad}_t \geq \estmu_{t}(\estell_t) + 3\texttt{rad}_t \ \geq \estmu_{t}(\ell^{*}) + 3\texttt{rad}_t \geq \mu(\ell^{*}),
    \end{align*}
    and thus 
    \[\mu(\ell^*) - \mu(\estell_t) \leq 6 \cdot \texttt{rad}_t.\]
    
\end{proof}

\begin{lemma}\label{lem:regret-per-threshold}
    Under perfect alignment, consider a selection algorithm that for context $H_t, B_t$ picks a function $\bar{\ell}_t \in \Lcal$ and selects decision $A_t=\II[B_t > \bar{\ell}_t(H_t)]$. Let $P(\estell_t)$ be the distribution over thresholds and $P(A_t \given H_t,B_t)$ be the distribution over decisions induced by this selection algorithm and distributions $P(Y\given H,B)$ and $P(H,B)$ at time step $t$ (these distributions are relevant as they influence the history $(H_i,B_i,Y_i)_{i< t}$ that the algorithm observes). Then, we have that
    \begin{multline}
        \EE[R(T)]= \EE_{h_t,b_t\sim P(H,B), a_t\sim P(A_t \given H_t,B_t)}\left[\sum_{t=1}^T \mu(\pi^*(h_t,b_t) \given h_t,b_t)- \mu(a_t \given h_t,b_t)\right]
        \\=\EE_{ \bar{\ell}_t \sim P( \bar{\ell}_t)}\left[\sum_{t=1}^T \mu(\ell^*)-\mu(\bar{\ell}_t)\right].
    \end{multline} 
\end{lemma}

\begin{proof}
    For any fixed function $\bar{\ell}$, the decision is defined deterministically by $\II[B_t > \estell(H_t)]$ given the fixed threshold $\bar{\ell}(H_t)$. Thus, we have that 
    \begin{align*}
        \EE_{h_t,b_t\sim P(H,B)} &\left[\mu(\II[b_t > \estell(h_t)] \given h_t,b_t) \right] \\&= 
        \EE_{h_t,b_t\sim P(H,B)} \left[\EE[\mu(1\given h_t,b_t) \cdot \mathbb{I}[ b_t> \bar{\ell}(h_t)] ] +\EE[\mu(0\given h_t,b_t) \cdot \mathbb{I}[ b_t\leq  \bar{\ell}(h_t) ]]\right] 
        \\& = \mu(\bar{\ell}) 
    \end{align*}
    Since, at each time step $t$, $\bar{\ell}_t$ is chosen by the selection algorithm  under distribution $P(\bar{\ell}_t)$, and the induced decision is $A_t= \II[B_t > \estell(H_t)]$, we can substitute the induced decision  $A_t \sim P(A \given H_t=h_t,B_t=b_t)$ for the function $\estell_t \sim P(\bar{\ell}_t)$ and obtain
    \begin{align*}
        \EE_{h_t,b_t\sim P(H,B), a_t\sim P(A_t \given H_t=h_t,B_t=b_t)} &\left[\mu(a_t \given h_t,b_t) \right] 
        \\&= \EE_{ \estell_t\sim P(\bar{\ell}_t)} [\EE_{h_t, b_t\sim P(H, B)}\left[ \mu(\II[b_t > \estell(h_t)] \given h_t,b_t)  \right]
        \\&= \EE_{ \bar{\ell}_t\sim P(\bar{\ell}_t)}[\mu(\bar{\ell}_t)]
    \end{align*}
    
    Since by Theorem~\ref{th:policy-reward-2d} the optimal policy $\pi^*$ correspond to a fixed optimal function $\ell^* \in \Lcal$, we get
    \[
    \EE_{h_t,b_t\sim P(H,B)} \left[\mu(\pi^*(h_t,b_t) \given h_t,b_t) \right] =
    \EE_{h_t,b_t\sim P(H,B)} \left[\mu(A^* \given h_t,b_t)  \right] 
    = \mu(\ell^*)
    \]
    where optimal decision $A^*$ depends on optimal function $\ell^*$ with $A^*= \II[b_t>\ell^*(h_t)]$.
    Thus, by linearity of expectation,  we have that
    \begin{align*}
        \EE[R(T)] 
        = \sum_{t=1}^T \EE_{ \bar{\ell}_t\sim P(\bar{\ell}_t)}[\mu(\ell^*) - \mu(\bar{\ell}_t)]
        = \EE_{ \bar{\ell}_t\sim P(\bar{\ell}_t)}\left[ \sum_{t=1}^T \mu(\ell^*) - \mu(\bar{\ell}_t)\right]
    \end{align*}

\end{proof}
With the above auxiliary Lemmas, we are ready to prove Theorem~\ref{th:regret-bound-2d}.

    From Lemma~\ref{lem:regret-per-threshold}, we have that 
    \[
    \EE[R(T)] =\EE_{ \bar{\ell}_t \sim P( \bar{\ell}_t)}\left[\sum_{t=1}^T \mu(\ell^*)-\mu(\bar{\ell}_t)\right].
    \]
    We show in the following that the expected regret of the multiclass online contextual learning problem (right hand side of the equation) is in $\mathcal{O} \left (\sqrt{T\log T} \right)$.
    Let $\estell_t = \argmax_{\ell \in \Lcal} \estmu(\ell)$, and consider the following event
    \begin{equation}
        \Ecal = \{\mu(\ell^*) - \mu(\estell_t)\leq 6 \cdot \sqrt{4c_2 \sqrt{|\Hcal|} \log(3T^3)/t}, \ \forall 1\leq t\leq T\}
    \end{equation}
    Then, using Lemma~\ref{lem:error-bound-2d} and an union bound, we can bound the probability of the complimentary event $\Ecal^c$ in terms of $T$, that is
    \begin{equation}
        P(\Ecal^c) \leq \sum_{t\leq T} P\left(\mu(\ell^*) - \mu(\estell_t)> 6 \cdot \sqrt{4c_2 \sqrt{|\Hcal|} \log(3T^3)/t}\right) = \frac{1}{T^2} 
    \end{equation}
    when $\log(3T^3)>c_3\sqrt{|\Hcal|}$, that is, when $\sqrt{|\Hcal|}< c_4\log(T)$ for some $c_4>0$ ($c_4$ equals constant $c$ in the statement of the theorem).
    
    We can now rewrite and bound the expected utility conditional on this event as
    \begin{align*}
        \EE[R(T)] &= P(\Ecal) \cdot \EE_{ \bar{\ell}_t \sim P( \bar{\ell}_t)}\left[\sum_{t=1}^T \mu(\ell^*)-\mu(\bar{\ell}_t) \given \Ecal \right] + 
        P(\Ecal^c) \cdot \EE_{ \bar{\ell}_t \sim P( \bar{\ell}_t)}\left[\sum_{t=1}^T \mu(\ell^*)-\mu(\bar{\ell}_t) \given \Ecal^c \right] 
        \\ &\leq \EE_{ \bar{\ell}_t \sim P( \bar{\ell}_t)}\left[\sum_{t=1}^T \mu(\ell^*)-\mu(\bar{\ell}_t) \given \Ecal \right] + P(\Ecal^c) \cdot T
        \\ &\leq \EE_{ \bar{\ell}_t \sim P( \bar{\ell}_t)}\left[\sum_{t=1}^T \mu(\ell^*)-\mu(\bar{\ell}_t) \given \Ecal \right] + \frac{1}{T}.
    \end{align*}
    The first term of the bound is bounded given the definition of event $\Ecal$, \ie,
    \[
    \EE_{ \bar{\ell}_t \sim P( \bar{\ell}_t)}\left[\sum_{t=1}^T \mu(\ell^*)-\mu(\bar{\ell}_t) \given \Ecal \right]
    \leq \sum_{t=1}^T  6 \cdot \sqrt{4c_2 \sqrt{|\Hcal|} \log(3T^3)/t} 
    \leq 12 \cdot \sqrt{12c_2 T\sqrt{|\Hcal|} \log(3T)} 
    \]
    where the last inequality holds as $\sum_{t=1}^T 1/\sqrt{t} \leq 2 \sqrt{T}$.
    For $\sqrt{|\Hcal|}< c_4\log(T)$, we can conclude that
    \[
        \EE[R(T)] \leq 12 \cdot \sqrt{12c_2 T\sqrt{|\Hcal|} \log(3T)} + \frac{1}{T}  
        = \mathcal{O}(\sqrt{T \log(T)}).
    \]


\subsection{Proof of a DKW-like Inequality for Threshold Policies}
\label{app:dkw-like-inequality}

We first introduce two theorems that will form the basis to prove the extension of the DKW inequality for functions of the form $\Delta_{\vec{x}}(k,x) := \II[x\leq \vec{x}_k]$ for $\vec{x} \in \Xcal^\Kcal$.

\begin{theorem}[Talagrand’s inequality \cite{talagrand1996new,bousquet2003concentration,sen2018gentle}]\label{th:talagrand}
    Let $X_i$, $i\in [n]$ be independent $\Xcal$-valued random variables. Let $\Fcal$ be a (countable) class of measurable real-valued functions on $\Xcal$ such that $|| f ||_{\infty}\leq U <\infty $  and $\EE[f (X_1)] = . . . = \EE[f (X_n)] =0$, for all $f\in \Fcal$. Let
    \[ Z= \sup_{f\in\Fcal } \sum_{i\in [n]} f(X_i) \qquad \text{or} \qquad Z= \sup_{f\in\Fcal } \left|\sum_{i\in [n]} f(X_i)\right|\]
    and let parameters $\sigma^2$ and $\nu_n$ be defined as
    \[ U^2 \geq \sigma^2 \geq \frac{1}{n} \sum_{i\in[n]} \sup_{f\in\Fcal} \EE[f^2(X_i)] \qquad \text{and} \qquad
        \nu_n = 2U\EE[Z] + n\sigma^2.
    \]
    Then, for all $\epsilon \geq 0$
    \begin{equation} \label{eq:talagrand}
        P(Z\geq \EE[Z]+\epsilon) \leq \exp\left(- \frac{\nu_n}{U^2} h_1 \left(\frac{\epsilon U}{\nu_n}\right) \right) \leq \exp\left( \frac{-\epsilon^2}{2\nu_n + 2 \epsilon U/3}\right) 
    \end{equation}
    where $h_1(w) = (1+w) \log(1+w) - w$  for $w\in \RR$ and
    \begin{equation*}
        P(Z\geq \EE[Z]+\sqrt{2\nu_n \delta}+ U\delta/3) \leq e^{-\delta}, \qquad \delta \geq 0.
    \end{equation*}
\end{theorem}

\begin{definition}[VC Dimension~\cite{sen2018gentle}]\label{def:vc-dimension}
    Given a function class $\Fcal$ of binary-valued functions, we say that the set $\{x_1,\dots, x_n\}$ is shattered by $\Fcal$ if
    \[
    |\Fcal(x_1,\dots, x_n)|=2^n \quad \text{where} \quad \Fcal(x_1,\dots, x_n) = \{ (f(x_1), \dots, f(x_n)) \given f\in \Fcal \},
    \]
    \ie, for any binary vector $\vec{u}$ of length n there exists a function $f \in \Fcal$, so that $(f(x_1), \dots, f(x_n))=\vec{u}$.
    The VC dimension $V(\Fcal)$ of $\Fcal$ is defined as the largest integer $n$ for which there is some collection $x_1,\dots, x_n$ of $n$ points that can be shattered by $\Fcal$.
\end{definition}

\begin{definition}[Envelope function~\cite{sen2018gentle}]
An envelope function of a class $\Fcal$ of functions is any function $F(\cdot)$ such that $|f(x)|\leq F(x)$, for every  $x \in \Xcal$ and $f \in \Fcal$.
\end{definition}

\begin{theorem}[\cite{sen2018gentle}]~\label{th:expectation-bound-VC-dimension}
    Let $X_i$, $i\in [n]$ be independent $\Xcal$-valued random variables with distribution $P(X)$. Let $\Fcal$ be a class of measurable real-valued functions on $\Xcal$ with envelope $F$ and VC dimension $V(\Fcal)$. Then, for some constant $C > 0$,
    \begin{equation}
        \EE\left[\sup_{f\in\Fcal} \left| \frac{1}{n}\sum_{i\in [n]} f(X_i)-\EE_{X\sim P(X)}(f(X))\right| \right] \leq C ||F||_{P,2} \sqrt{\frac{V(\Fcal)}{n}}
    \end{equation}
    where $||F||_{P,2} = \left(\int F^2(x)\cdot P(X=x) dx \right)^{1/2} $.
\end{theorem}

Next, we formally describe the class of functions we are interested in. 
\begin{definition}\label{def:def-2Dthresholds}
    Let $\Xcal$ and $\Kcal$ be countable, totally ordered sample spaces. Given a fixed vector $\vec{x}\in \Xcal^\Kcal$, $\vec{x} = (\vec{x}_k)_{k\in \Kcal}$, let $\Delta_{\vec{x}}$ be the following function for points $(k, x) \in \Kcal \times \Xcal$
\[\Delta_{\vec{x}}(k,x) := \II[ x\leq \vec{x}_k].\]
We define $\Dcal$ as the class of all such functions over $\Xcal$, \ie,
\begin{equation*}\label{eq:def-2Dthresholds}
    \Dcal := \{ \Delta_{\vec{x}} \given \vec{x} \in \Xcal^\Kcal\}
\end{equation*}
We also define the complementary class of functions denoted by $\Dcal^+$ as
\begin{equation*}
    \Dcal^+ := \{ 1-\Delta \given \Delta \in \Dcal \}
\end{equation*}
This class is simply the class of functions of the type $\Delta^+_{\vec{x}}(k,x):=\II[x>\vec{x}_k]$ for $\vec{x}_k \in \Xcal$, \ie,
    $\Dcal^+ = \{ \Delta^+_{\vec{x}} \given \vec{x} \in \Xcal^\Kcal\}.$
\end{definition}

\begin{lemma}\label{lem:prelim-dkw-2d}
    Let $(K_i, X_i)$, $i\in [n]$ be independent and identically distributed random variables with distribution $P(K,X)$ over sample space $\Kcal \times \Xcal$ and
    consider the following (countable) class of functions
    \[
    \Fcal = \{ \Delta(\cdot)-\EE_{K,X\sim P(K,X)}[\Delta(K, X)] \given \Delta \in \Dcal \}
    \] where $\Dcal$ is the class of functions as in Definition~\ref{def:def-2Dthresholds} and let
    \[ Z := n \sup_{\Delta \in \Dcal} \left| \frac{1}{n}\sum_{i\in [n]} \Delta(K_i, X_i)-\EE_{K,X\sim P(K,X)}[\Delta(K, X)]\right|.\]
    Then, 
    \begin{enumerate}[label=\alph*)]
        \item $Z= \sup_{f\in\Fcal } \left|\sum_{i\in [n]} f(X_i)\right|$ \label{stat:a}
        \item $||f||_{\infty}= \sup_{k,x\in \Kcal\times \Hcal} |f(k,x)| \leq 1$ for all $f\in \Fcal$ \label{stat:b}
        \item $\EE[f(K_i,X_i)] = 0$ for all $f\in \Fcal$, $i\in[n]$ \label{stat:c}
        \item $\EE[f^2(K_i,X_i)] \leq 1/4 $ for all $f\in \Fcal$, $i\in[n]$ \label{stat:d}
        \item $||\mathbf{\Delta}||_{P,2} \leq 1$ where $\mathbf{\Delta}$ is the envelope of $\Dcal$ \label{stat:e}
        \item $V(\Dcal) \leq |\Kcal|$ \label{stat:f}
        \item $\EE[Z]\leq C \sqrt{|\Kcal|n}$ for some constant $C$ \label{stat:g}
    \end{enumerate}
\end{lemma}
\begin{proof}
    We will show each statement in order. Statement~\ref{stat:a} follows from simple arithmetic manipulations and by substituting the supremum over class $\Dcal$ by the supremum over class $\Fcal$, that is
    \begin{align*}
      Z &= n \sup_{\Delta \in \Dcal} \left| \frac{1}{n}\sum_{i\in [n]} \Delta(K_i, X_i)-\EE_{K,X\sim P(K,X)}(\Delta(K, X))\right| 
      \\& = n \sup_{\Delta \in \Dcal} \left| \frac{1}{n}\sum_{i\in [n]} \Delta(K_i, X_i)- \frac{1}{n} n \cdot \EE_{K,X\sim P(K,X)}(\Delta(K, X))\right| 
      \\& = \sup_{\Delta \in \Dcal} \left| \sum_{i\in [n]} \left[ \Delta(K_i, X_i)-\EE_{K,X\sim P(K,X)}(\Delta(K, X)) \right] \right|
      \\&= \sup_{f\in\Fcal } \left|\sum_{i\in [n]} f(X_i)\right|
    \end{align*}
    To show statement~\ref{stat:b}, note that all $\Delta \in \Dcal$ are binary-valued functions. Thus, $f(k,x) \in [-1,1]$ for all $f \in \Fcal$ and $k,x \in \Kcal \times \Xcal$ which implies $||f||_{\infty} \leq 1$ for all $f \in \Fcal$.
    
    Statement~\ref{stat:c} can be derived as follows. Let $f$ be any function in $\Fcal$, then using the linearity of expectation we have that
    \begin{align*}
        \EE[f(K_i,X_i)] &= \EE_{K_i,X_i\sim P(K,X)}[\Delta(K_i,X_i) - \EE_{K,X\sim P(K,X)}[\Delta(K,X)] ] 
        \\ &= \EE_{K_i,X_i\sim P(K,X)}[\Delta(K_i,X_i)] - \EE_{K,X\sim P(K,X)}[\Delta(K,X)] 
        =0
    \end{align*}
        
    Statement~\ref{stat:d} can be derived as follows. Let $f$ be any function in $\Fcal$, then we have that
    \begin{align*}
        \EE[f^2(K_i,X_i)] &= \EE_{K_i,X_i\sim P(K,X)}[(\Delta(K_i,X_i) - \EE_{K,X\sim P(K,X)}[\Delta(K,X)])^2 ] 
        \\ &= \EE_{K_i,X_i\sim P(K,X)}[\Delta(K_i,X_i)^2] +\EE_{K,X\sim P(K,X)}[\Delta(K,X)]^2
        \\ & \qquad \qquad- 2 \cdot\EE_{K_i,X_i\sim P(K,X)}[\Delta(K_i,X_i)]\cdot \EE_{K,X\sim P(K,X)}[\Delta(K,X)] 
        \\&= \EE_{K,X\sim P(K,X)}[\Delta(K,X)^2] - \EE_{K,X\sim P(K,X)}[\Delta(K,X)]^2
        \\&= \EE_{K,X\sim P(K,X)}[\Delta(K,X)] - \EE_{K,X\sim P(K,X)}[\Delta(K,X)]^2
    \end{align*}
    where the last equation follows because $\Delta(K,X)^2 = \Delta(K,X)$ as $\Delta(K,X)\in\{0,1\}$. Note that, since $\EE_{K,X\sim P(K,X)}[\Delta(K,X)] \in [0,1]$, it holds that
    \[\EE[f^2(K_i,X_i)] = \EE_{K,X\sim P(K,X)}[\Delta(K,X)] ( 1- \EE_{K,X\sim P(K,X)}[\Delta(K,X)]) \leq 1/4\]
    for all $f\in \Fcal$ (with equality satisfied for $\Delta$ such that $\EE_{K,X\sim P(K,X)}[\Delta(K,X)]=1/2$).
    
    To show statement~\ref{stat:e}, first note that we can take the envelope of function class $\Dcal$ to be the constant function $\mathbf{1}$ (as $\Dcal$ is a class of binary-valued functions).
    It is straightforward to see that 
    \[||\mathbf{1}||_{P,2}= \left(\int 1^2\cdot P(K=k,X=x)dkdx\right)^{1/2} = 1.\]

    Statement~\ref{stat:f} can be derived by contradiction.
    First, note that for a class of simple threshold functions on $\Xcal$ the VC-dimension is $1$. More formally, for class
    \[
        \Gcal = \{ g_{\theta}(x) := \II[x \leq \theta ] \given \theta \in \Xcal \},
    \]
    and any two points $(x_1, x_2) \in \Xcal^2$, if the larger point is labeled $1$ by a function in $\Gcal$, then the smaller point must also be labeled $1$, so not all labelings are possible (\ie, only $(0,1)$, if $x_1 > x_2$, or $(1,0)$, if $x_1 \leq x_2$, is possible but not both). Thus, we have that
    \[
    |\{(g(x_1),g(x_2)) \given g\in \Gcal\}| < 2^2,
    \]
    that is no two points $(x_1, x_2) \in \Xcal^2$ can be shattered by $\Gcal$. By the same argument no larger sets can be shattered by $\Gcal$ and we have that $V(\Gcal)\leq 1$ (see Definition~\ref{def:vc-dimension}).
    Next, assume that $V(\Dcal)=m >|\Kcal|$. Then, there must exist a set of points $((k_1,x_1), \dots, (k_m,x_m))$ such that 
    \begin{equation}\label{eq:vc-contradiction}
        |\{(\Delta(k_1,x_1), \dots, \Delta(k_m,x_m)) \given \Delta \in \Dcal \} | =2^m 
    \end{equation}  
    By the pigeon hole principle, there must exist $k \in \Kcal$ such that $|\Ical| \geq 2$, where $\Ical := \{i\in [m] \given k_i=k\}$, and
    \[|\{(\Delta(k,x_i))_{i\in \Ical} \given \Delta \in \Dcal \}| = 2^{|\Ical|}\geq 2^2\]
    as otherwise Eq.~\ref{eq:vc-contradiction} would not hold.
    This implies $V(\{\Delta(k,\cdot)\given \Delta \in \Dcal\})\geq 2 $. However, since $\{\Delta(k,\cdot)\given \Delta \in \Dcal\} \subseteq \Gcal$, it follows that $V(\Gcal)\geq 2$ which is a contradiction. 
    Hence, $V(\Dcal)\leq |\Kcal|$.
    
    Statement~\ref{stat:g} can be derived by using Theorem~\ref{th:expectation-bound-VC-dimension} together with statement~\ref{stat:e} and~\ref{stat:f},
    \begin{multline*}
        \EE[Z] = n \cdot \EE\left[ \sup_{\Delta \in \Dcal} \left| \frac{1}{n}\sum_{i\in [n]} \Delta(K_i, X_i)-\EE_{K,X\sim P(K,X)}[\Delta(K, X)]\right| \right] \\ \leq n\cdot C ||\mathbf{\Delta}||_{P,2} \sqrt{\frac{V(\Dcal)}{n}} \leq C \sqrt{|\Kcal| n}
    \end{multline*}
\end{proof}

\begin{theorem}\label{thm:dkwi-thresholds}
    Let $(K_1, X_1), \dots, (K_n,X_n)$
    be $n$ independent and identically distributed random variables with distribution $P(K,X)$ over sample space $\Kcal \times \Xcal$ and let $\Dcal$ be the (countable) class of functions as in Definition~\ref{def:def-2Dthresholds}. Then, there exist constant $c_1>0$ such that for any $\epsilon> c_1\sqrt{|\Kcal|/n}$
    \[
    P\left(\sup_{\Delta \in \Dcal} \left| \frac{1}{n} \sum_{i\in [n]} \Delta(K_i, X_i) - \EE_{K,X\sim P(K,X)}[\Delta(K, X)]\right| > \epsilon \right) \leq \exp\left(-\frac{n\epsilon^2}{4c_2\sqrt{|\Kcal|}} \right)
    \]
    for some constant $c_2>c_1$.
\end{theorem}
\begin{proof}
    This proof closely follows the proof of the DKW inequality for outcome space $\RR^d$ outlined in \cite{sen2018gentle}. Let $Z$ be as in Lemma~\ref{lem:prelim-dkw-2d}.
    We have to upper-bound the probability
    \[
    P\left( \frac{Z}{n} >\epsilon\right) = P( Z> n\epsilon) = P(Z-\EE[Z] > n\epsilon-\EE[Z] )
    \]
    Note that for $\epsilon > 2$, this probability is zero (as $Z \leq 2n$).
    For $\epsilon> c_1\sqrt{|\Kcal|/n}$ with $c_1:=2C$ as in Lemma~\ref{lem:prelim-dkw-2d}~\ref{stat:g}, then $n\epsilon-\EE[Z] \geq (n\epsilon-c_1 \sqrt{|\Kcal|n})>0$ and we can apply Theorem~\ref{th:talagrand} (other prerequisites shown in Lemma~\ref{lem:prelim-dkw-2d}~\ref{stat:a}-~\ref{stat:d}). Hence, for $2\geq \epsilon> c_1\sqrt{|\Kcal|/n}$,
    \begin{equation}\label{eq:talagrand-bound}
        P(Z> \EE[Z] + (n\epsilon-\EE[Z]) )\leq \exp\left( \frac{-(n\epsilon-\EE[Z])^2}{2\nu_n + 2 (n\epsilon-\EE[Z]) U/3}\right).
    \end{equation}
    We can use Lemma~\ref{lem:prelim-dkw-2d}~\ref{stat:b} and~\ref{stat:d} to obtain bounds $U\leq 1$ and $\nu_n \leq 2c_1\sqrt{|\Kcal|n} + n/4$, inserting it in Eq.~\ref{eq:talagrand-bound} we get 
    \begin{align}
        P(Z> \EE[Z] + (n\epsilon-\EE[Z]) ) &\leq \exp\left( -\frac{(n\epsilon-\EE[Z])^2}{4C\sqrt{|\Kcal|n} + n/2 + \frac{2}{3} (n\epsilon-\EE[Z])}\right) \label{eq:exp-1}
        \\& \leq \exp\left( -\frac{(n\epsilon-C\sqrt{|\Kcal|n})^2}{c_2 n \sqrt{|\Kcal|}}\right) 
        = \exp\left( -\frac{n(\epsilon-C\sqrt{|\Kcal|/n})^2}{c_2 \sqrt{|\Kcal|}}\right) \nonumber
        \\& \leq \exp\left( -\frac{ n\epsilon^2}{4 c_2\sqrt{|\Kcal|}}\right) \label{eq:exp-2}
    \end{align}
    where for $\epsilon\leq 2$ the denominator in Eq.~\ref{eq:exp-1} is upper bounded by $4 C \sqrt{|\Kcal|n}+ n/4 + 4n/3$ which is in turn upper bounded by $c_2 \sqrt{|\Kcal|}n$ (for some $c_2 > c_1$) and 
    Eq.~\ref{eq:exp-2} holds as for $\epsilon > 2C\sqrt{|\Kcal|/n}$, $(\epsilon - C\sqrt{|\Kcal|/n})^2 \geq (\epsilon - \epsilon/2)^2 =\epsilon^2/4 $.

    Thus, there exists constant $c_1>0$ such that for any $\epsilon> c_1\sqrt{|\Kcal|/n}$
    \begin{multline*}
        P\left(\sup_{\Delta \in \Dcal} \left| \frac{1}{n} \sum_{i\in [n]} \Delta(K_i, X_i) - \EE_{K,X\sim P(K,X)}[\Delta(K, X)]\right| > \epsilon \right) 
        \\= P(Z > n\epsilon) \leq \exp\left(-\frac{n\epsilon^2}{4c_2\sqrt{|\Kcal|}} \right)
    \end{multline*}
    for some constant $c_2>c_1$.
\end{proof}

\begin{corollary}\label{cor:dkwi-thresholds+}
    Given a natural number $n$, let $(K_1, X_1), \dots, (K_n,X_n)$
    be independent and identically distributed random variables with distribution $P(K,X)$ over sample space $\Kcal \times \Xcal$ and let $\Dcal^+$ be the (countable) class of functions as in Definition~\ref{def:def-2Dthresholds}. Then, there exist constant $c_1>0$ such that for any $\epsilon> c_1\sqrt{|\Kcal|/n}$
    \[
    P\left(\sup_{\Delta^+ \in \Dcal^+} \left| \frac{1}{n} \sum_{i\in [n]} \Delta^+(K_i, X_i) - \EE_{K,X\sim P(K,X)}[\Delta^+(K, X)]\right| > \epsilon \right) \leq \exp\left(-\frac{n\epsilon^2}{4c_2\sqrt{|\Kcal|}} \right)
    \]
    for some constant $c_2>c_1$.
\end{corollary}
\begin{proof}
    First, note that, for any $\Delta \in \Dcal$, we have that
    \begin{align}\label{eq:dkwi-threshold-expression}
        \bigg| \frac{1}{n} \sum_{i\in [n]} \Delta(K_i, X_i)  &-\EE_{K,X\sim P(K,X)}[\Delta(K, X)]\bigg|
     \\&= \left| - 1  +\frac{1}{n} \sum_{i\in [n]} \Delta(K_i, X_i))  + 1- \EE_{K,X\sim P(K,X)}[\Delta(K, X)]\right| 
     \\ &= \left| - \left( \frac{1}{n} \sum_{i\in [n]} 1- \Delta(K_i, X_i)) \right) + \EE_{K,X\sim P(K,X)}[1-\Delta(K, X)]\right|.
    \end{align}
    By definition of $\Dcal$ and $\Dcal^+$, we also have that $\Delta \in \Dcal \iff 1-\Delta \in \Dcal^+$. Hence, the supremum of the left expression in Eq.~\ref{eq:dkwi-threshold-expression} over $\Dcal$ and the right expression over $\Dcal^+$ are the same. Using this fact and Theorem~\ref{thm:dkwi}, we get that
    \begin{align*}
        &P\left(\sup_{\Delta^+ \in \Dcal^+} \left| \frac{1}{n} \sum_{i\in [n]} \Delta^+ (K_i, X_i) - \EE_{K,X\sim P(K,X)}[\Delta^+ (K, X)]\right| > \epsilon \right) 
        \\= &P\left(\sup_{\Delta \in \Dcal} \left| \frac{1}{n} \sum_{i\in [n]} \Delta(K_i, X_i) - \EE_{K,X\sim P(K,X)}[\Delta(K, X)]\right| > \epsilon \right) \leq \exp\left(-\frac{n\epsilon^2}{4c_2\sqrt{|\Kcal|}} \right).
    \end{align*}
\end{proof}

\clearpage
\newpage

\section{Learning to Decide with AI Assistance under Imperfect Alignment} \label{app:approximate-alignment}

In this section, we show that, under imperfect alignment, the decision policy found by Algorithm~\ref{alg:bandit-alg} is near-optimal, and the degree of alignment, as measured by the maximum alignment error (MAE), bounds the near-optimality gap.

To this end, we begin with the observation that the expected utility achieved by the decision policy $\pi(h, b) = \II[b > b^{*}(h)]$, defined in Theorem~\ref{th:optimal-threshold}, matches that achieved by the optimal decision policy $\pi^{*}_{\texttt{mono}}$ within the set of decision policies monotone in the AI confidence.\footnote{A decision policy $\pi$ is monotone in the AI confidence if, for any $h \in \Hcal$ and any $b, b' \in \Hcal$ such that $b \leq b'$, it holds that $\pi(h, b) \leq \pi(h, b')$.}
This is because any policy $\pi_{\texttt{mono}}$ monotone in the AI confidence can be expressed as a threshold function $\pi_{\texttt{mono}}(h, b) = \II[b > b(h)]$, as shown elsewhere~\citep{corbett2017algorithmic}.

Based on the above observation and Theorem~5 in~\cite{corvelo2023human}, 
which characterizes the maximum gap between the expected utility of the optimal decision policy $\pi^{*}$ and the optimal decision policy within the set of decision policies monotone in both the human and AI confidence\footnote{A decision policy $\pi$ is monotone in the human and the AI confidence if, for any $h, h' \in \Hcal$ and $b, b' \in \Hcal$ such that $h \leq h'$ and $b \leq b'$, it holds that $\pi(h, b) \leq \pi(h', b')$.},
it readily follows that
\begin{multline}~\label{eq:exp-utility-gap}
    \EE_{H,B \sim P(H,B)}[\mu(\pi^*(H,B) \given H,B)] - \EE_{H,B \sim P(H,B)}[\mu(\pi^*_{\texttt{mono}}(H,B) \given H,B)] 
    \\ \leq \text{MAE} \cdot \left[ u(1,1) - u(0,1) + \frac{3}{2}(u(0,0)-u(0,1))\right],
\end{multline}
because the set of decision policies that are monotone with respect to the AI confidence includes the set of decision policies that are monotone with respect to the human and the AI confidence.

Finally, since Algorithm~\ref{alg:bandit-alg} is guaranteed to learn $\pi^{*}_{\texttt{mono}}$ as $T$ grows, as shown in Theorems~\ref{th:regret-bound} and~\ref{th:regret-bound-2d}, we can conclude that Algorithm~\ref{alg:bandit-alg} is near-optimal under imperfect alignment.

\clearpage
\newpage

 \section{Additional Experimental Details}
 \label{app:experiments}
We provide the MAE and EAE for each group in the Human-Alignment dataset in Table~\ref{tab:human-alignment-dataset} and each task in the Human-AI Interactions dataset in Table~\ref{tab:haii-dataset}. The Human-Alignment dataset is publicly available under GNU General Public License v3.0 and Human-AI Interactions dataset is publicly available under MIT License. Further, we clarify that in the Human-Alignment dataset the game-specific parameter $q$  is noted as $r$ in~\citet{corvelo2025human} and represents the fraction of red cards in the card pile of the game instance. For more details and the exact card game procedures, refer to~\citet{corvelo2025human}. 

We have adapted two classical algorithms to exploit the monotonicity of the outcome probabilities
under perfect alignment, namely, online isotonic regression (Algorithm~\ref{alg:iso-alg}) and follow-the-leader (Algorithm~\ref{alg:ftl-alg}). 
Algorithm~\ref{alg:iso-alg} uses isotonic regression to fit an estimator $\texttt{isotonic-regression}(\Dcal_{h_t}, b_t)$ for the probability of $Y=1$ given $b_t$ using the observed AI confidences and outcomes for $h_t$, $\Dcal_{h_t}=\{ (b_i, y_i )\given 0<i<t \text{ and } h_i=h_t\}$. This estimate is used to make decision $a_t$ according to Lemma~\ref{lem:optimal-action-fixed-context}.
Algorithm~\ref{alg:ftl-alg} keeps track of the cumulative rewards $\text{reward}_t(h,b)$ of each threshold $b \in \Bcal$ for each fixed $h \in \Hcal$. At round $t$, the threshold $\bar{b}$ with highest cumulative reward for $h_t$ is chosen setting decision $a_t = \II[b_t>\bar{b}]$.
Finally, below we present further details on the implementation of our experiment.

\xhdr{Implementation details} Our experiments ran with \texttt{Python 3.13.5}, mainly using the open-source libraries \texttt{numpy} (BSD-3-Clause License) and \texttt{pandas} (BSD-3-Clause License). We provide the exact versions of all used libraries with our code as supplementary material. Our experiments ran on one node of an internal cluster of Debian machines with 2 AMD EPYC 7702 $64$-core  CPU processors at $2.00$GHz and $2$TB of memory with total execution time below five minutes. 

 \begin{table}[h]
    \centering
    \begin{tabular}{c|r|c|c}
        Group  & \centering{EAE}  & MAE & \# data samples\\
        \hline
         A & $6.5 \cdot 10^{-4}$ &  $0.10$ & 2{,}040\\
         B & $69.3\cdot 10^{-4}$ & $0.20$ & 2{,}014\\
         BP & $23.6\cdot 10^{-4}$ & $0.12$  & 2{,}184\\
         C & $3\cdot 10^{-4}$ & $0.06$  & 1{,}896\\
    \end{tabular}
    \caption{Expected Alignment Error (EAE) and Maximum Alignment Error (MAE) as in~\citet{corvelo2023human} of the AI model within each participant group in the Human-Alignment dataset together with total number of data samples collected from participants in each group. }
    \label{tab:human-alignment-dataset}
\end{table}

\begin{table}[h]
    \centering
    \begin{tabular}{c|r|c|c}
        Task & EAE & MAE & \# data samples \\
        \hline
        Art & $9{.}2 \cdot 10^{-4}$ & $0{.}108$ & $4{,}701$\\
        Census & $298.4 \cdot 10^{-4}$ & $0{.}673$ & $2{,}941$\\
        Cities & $2{.}7 \cdot 10^{-4}$ & $0{.}065$ & $4{,}543$ \\
        Sarcasm & $9{.}7  \cdot 10^{-4}$ & $0{.}130$ & $2{,}878$  \\
    \end{tabular}
    \caption{Expected Alignment Error (EAE) and Maximum Alignment Error (MAE) as in~\citet{corvelo2023human} of the AI model used in each task in the Human-AI Interactions dataset together with total number of data samples in each task.}
    \label{tab:haii-dataset}
\end{table}

\begin{algorithm}[t]
\caption{Online isotonic regression for online learning under perfect alignment}\label{alg:iso-alg}
\begin{algorithmic}
    \Require  \# interactions $T$, set of human confidence values $\Hcal$, set of AI confidence values $\Bcal$.
    \State{Initialize $\Dcal_{h} \leftarrow \emptyset$, $\forall h \in \Hcal$} 
    \For{$t$ from $1$ to $T$} 
    \State{Observe confidence values $(h_t, b_t)$ }
    \If{$\Dcal_{h_t} = \emptyset$}
    \State $\bar{p}_t(Y=1|H=h_t,B=b_t) \leftarrow 0.5$
    \Else
    \State $\bar{p}_t(Y=1|H=h_t,B=b_t) \leftarrow \texttt{isotonic-regression}(\Dcal_{h_t}, b_t)$
    \EndIf
    \If{$1-\bar{p}_t(Y=1|H=h_t,B=b_t) < \frac{u(1,1)-u(0,1)}{u(1,1)-u(1,0) + u(0,0) -u(0,1)}$}
    \State $a_t \leftarrow 1$
    \Else
    \State $a_t \leftarrow 0$
    \EndIf
    \State Take decision $a_t$ and observe label value $y_t$
    \State Add $(b_t, y_t)$ to set $\Dcal_{h_t}$
    \EndFor
\end{algorithmic}
\end{algorithm}

\begin{algorithm}[t]
\caption{Follow-the-leader algorithm for online learning under perfect alignment}\label{alg:ftl-alg}
\begin{algorithmic}
    \Require  \# interactions $T$, set of human confidence values $\Hcal$, set of AI confidence values $\Bcal$.
    \State{Initialize $\text{reward}_1(h,b) \leftarrow 0$, $\forall (h, b) \in \Hcal \times \Bcal$} 
    \For{$t$ from $1$ to $T$} 
    \State{Observe confidence values $(h_t, b_t)$ }
    \State{Compute leader $\bar{b} \leftarrow \arg\max_{b \in \Bcal}  \text{reward}_t(h_t,b)$}
    \If{$b_t > \bar{b}$}
    \State $a_t \leftarrow 1$
    \Else
    \State $a_t \leftarrow0$
    \EndIf
    \State Take decision $a_t$ and observe label value $y_t$
    \For{each $b \in \Bcal$} 
    \State $\text{reward}_{t+1}(h_{t+1},b) \leftarrow \text{reward}_t(h_t,b) + u(\II[b_t>b], y_t)$
    \EndFor
    \EndFor
\end{algorithmic}
\end{algorithm}

\clearpage
\newpage

\end{document}